\documentclass[journal]{IEEEtran}
\usepackage[T1]{fontenc}
\usepackage[utf8]{inputenc}
\usepackage{array}
\usepackage{float}
\usepackage{multirow}
\usepackage{amsmath}
\usepackage{amsthm}
\usepackage{amssymb}
\usepackage{graphicx}
\usepackage[bookmarks=true,bookmarksnumbered=true,bookmarksopen=true,bookmarksopenlevel=1,
 breaklinks=false,pdfborder={0 0 0},pdfborderstyle={},backref=false,colorlinks=false]
 {hyperref}
\hypersetup{pdftitle={Your Title},
 pdfauthor={Your Name},
 pdfpagelayout=OneColumn, pdfnewwindow=true, pdfstartview=XYZ, plainpages=false}

\makeatletter

\providecommand{\tabularnewline}{\\}
\floatstyle{ruled}
\newfloat{algorithm}{tbp}{loa}
\providecommand{\algorithmname}{Algorithm}
\floatname{algorithm}{\protect\algorithmname}

\usepackage[caption=false,font=footnotesize]{subfig}
\usepackage{algorithm}
\usepackage{algpseudocode}

\makeatother

\theoremstyle{plain}
\newtheorem{thm}{\protect\theoremname}
\newtheorem{cor}[thm]{\protect\corollaryname}
\newtheorem{lem}[thm]{\protect\lemmaname}
\providecommand{\corollaryname}{Corollary}
\providecommand{\lemmaname}{Lemma}
\providecommand{\theoremname}{Theorem}

\begin{document}

\title{Robust Decentralized Federated Distillation via Multi-Modality Knowledge
Collaboration}
\author{{\normalsize Xiao Ma$^{a}$, Hong Shen$^{b}$, Hui Tian$^{c}$, Wei
Ke$^{a}$, Wenqi Lyu$^{a}$}
\medskip

{\normalsize$^{a}$Faculty of Applied Sciences, Macao Polytechnic
University, Macao SAR, China}\\
{\normalsize$^{b}$School of Engineering and Technology, Central Queensland
University, Australia}\\
{\normalsize$^{c}$School of Information and Communication Technology,
Griffith University, Australia}}

\maketitle

\begin{abstract} 
Decentralized federated learning enables clients to collaborate without
centralized coordination, but most existing methods assume that all
clients use the same model architecture. This requirement limits their
use in practical systems where clients may adopt different models
because of heterogeneous computation and resource constraints. The
problem becomes more challenging under Byzantine attacks, since parameter-based
robust aggregation is ineffective for heterogeneous models, especially
for defense against receiver-specific attacks where malicious clients
may send different manipulated predictions to different honest clients.
To address this issue, we propose a robust decentralized federated
distillation method that enables clients with heterogeneous models
to collaborate through predictions on shared unlabeled public data.
In the proposed method, each client first evaluates the received predictions
in three modalities of class prediction, boundary decision, and prediction
correlation. It then filters unreliable clients, assigns reliability-based
weights to the retained clients, and constructs a teacher for each
type of knowledge. Finally, the corresponding distillation gradients
are validated using a supervised gradient computed from private data.
Conflicting prediction and boundary gradients are removed, and conflicting
relation gradients are suppressed before the final model update. We
prove the convergence of the proposed method by showing stable local
optimization for honest clients under Byzantine distillation. Particularly,
we show that our method ensures a bounded Byzantine influence on both
distillation gradients and individual client private gradients after
cross-modality fusion, thereby enabling stable local optimization
for honest clients under Byzantine distillation. Extensive experiments
on CIFAR-10 and CIFAR-100 demonstrate that the proposed method improves
the prediction accuracy of heterogeneous models of clients under non-IID
data and Byzantine attacks. As the booming demands of federated learning
in decentralized environments such as edge computing and mission-oriented
UAV collaborations, our method has a great potential for adoption
of DFL in unreliable real-world scenarios where clients are exposed
to receiver-specific Byzantine messages of malicious predictions.
\end{abstract}


\section{Introduction}

\IEEEPARstart{F}{ederated} learning (FL) enables multiple clients
to train models collaboratively while retaining their raw data locally.
Conventional FL, however, still relies on a central server to distribute
a global model and aggregate client updates\cite{pmlr-v54-mcmahan17a}.
This coordinator can become a communication bottleneck and a single
point of failure, and it may be unavailable or undesirable in collaborations
among autonomous organizations or peer devices. Decentralized federated
learning (DFL) removes this dependency by allowing clients to communicate
directly over a graph. Decentralized stochastic optimization can distribute
the communication load across neighboring nodes\cite{lian2017dpsgd},
while consensus-based FL demonstrates that serverless model optimization
can be implemented through cooperation among devices\cite{savazzi2020consensus}.
These properties make DFL attractive when no trusted server exists
or no participant should control a shared global model.

Most DFL algorithms nevertheless inherit a restrictive assumption
from parameter-aggregation FL: all clients optimize models with the
same architecture and an aligned parameter space\cite{ye2023heterogeneous}.
This assumption is difficult to maintain in practice because clients
may differ in computation, memory, energy, latency, and hardware support.
Consequently, they may deploy models with different depths, widths,
operators, and parameter dimensions. Under such model heterogeneity,
averaging parameters or gradients is undefined when their dimensions
differ. Even when two parameter vectors have the same dimension, coordinate-wise
comparison or Euclidean distance need not represent functional similarity
because distinct architectures encode knowledge differently. Prior
heterogeneous FL studies have therefore identified gradient misalignment
and architectural dependence as central limitations of parameter-space
aggregation\cite{tan2022fedproto,kalra2023proxyfl}. The same limitation
affects many robustness rules that judge a client by the distance
or coordinates of its transmitted model.

Knowledge distillation (KD) provides an alternative communication
space. Instead of transferring internal parameters, KD trains a student
model to match the predictive behavior of a teacher through soft outputs\cite{hinton2015distilling}.
This principle allows models with different architectures to exchange
task knowledge as long as they share a common input and label space.
FedMD applies prediction matching on public data to heterogeneous
federated models\cite{li2019fedmd}. In serverless networks, consensus-based
multi-hop federated distillation exchanges predictions on shared samples
to approximate collaboration in function space\cite{taya2022cmfd},
while SD-Dist uses model outputs on a shared reference dataset to
select useful peers\cite{ye2022sddist}. Compared with transmitting
a parameter vector whose modality depends on the local architecture,
transmitting logits on a shared public reference dataset provides
a compact, architecture-agnostic interface for decentralized collaboration.

Moving communication from parameter space to output space does not
remove adversarial risk; it changes the object that an adversary can
manipulate. A Byzantine client may transmit arbitrary logits, reverse
class preferences, amplify selected entries, or imitate plausible
but misleading predictions. Recent work on federated distillation
has shown that attacks tailored to logits can substantially disrupt
knowledge aggregation and that defenses designed for parameter updates
do not directly transfer to this setting\cite{li2024fedtgd,roux2025byzantinefd}.
The threat is more subtle in a peer-to-peer network. Without a server
that produces one common aggregate, a Byzantine client can send different
messages to different clients during the same training round. Such
equivocation is permitted by standard decentralized Byzantine models
and can drive honest receivers toward different references \cite{he2023clippedgossip}.
Thus, an honest client must assess the predictions available at its
own receiver rather than rely on a globally consistent view of the
network.

Existing research addresses only parts of this combined problem. Byzantine-robust
DFL methods typically screen or transform client model vectors in
a shared parameter space. BRIDGE performs coordinate-wise filtering
before decentralized mixing\cite{fang2022bridge}; UBAR combines parameter
distance with local performance evaluation\cite{guo2022ubar}; and
ClippedGossip and remove-then-clip aggregation bound the effect of
suspicious model differences\cite{he2023clippedgossip,yang2024removethenclip}.
These methods are effective when the exchanged vectors have compatible
semantics, but their coordinate and distance operations cannot be
directly applied to models with different architectures. Conversely,
decentralized distillation methods support heterogeneous models by
exchanging outputs, yet they are primarily designed for benign knowledge
transfer rather than receiver-specific Byzantine manipulation\cite{taya2022cmfd,ye2022sddist,kalra2023proxyfl}.
Robust output aggregation has been studied in server-coordinated collaborative
learning and federated distillation\cite{chang2019cronus,li2024fedtgd,roux2025byzantinefd},
but a central aggregate does not capture the local and potentially
inconsistent message sets observed by different clients in a serverless
network. This leaves an unresolved gap: how to achieve robust decentralized
knowledge collaboration when honest clients may use heterogeneous
models and observe non-IID data, while Byzantine clients can manipulate
the prediction messages delivered to different receivers.

To address this gap, we propose a robust decentralized federated distillation
method for collaborative learning with heterogeneous local models.
Instead of exchanging model parameters across heterogeneous architectures,
clients communicate only their predictions on a shared unlabeled public
dataset. Each honest client independently evaluates the received knowledge
from three complementary modalities: class prediction, boundary decision,
and prediction correlation. For each modality, unreliable clients
are filtered and weighted to construct a teacher. Before distillation,
the corresponding gradients are further validated using a supervised
gradient computed from private data, and conflicting directions are
projected or suppressed. This enables heterogeneous models to collaborate
under reduced influence of Byzantine predictions. Our method has a
great application value for decentralized federated learning in unreliable
real-world scenarios where clients are exposed to receiver-specific
Byzantine messages of malicious predictions.

The main contributions of this work are summarized as follows:
\begin{enumerate}
\item We propose a robust decentralized federated distillation method with
heterogeneous models under Byzantine attacks. By exchanging predictions
on shared public data rather than model parameters, our method enables
clients with different model architectures to collaborate without
a central server while allowing each honest client to independently
defend against malicious predictions.
\item We design a multi-modality knowledge collaboration mechanism with
private-gradient validation. Neighbor selection, reliability weighting,
and teacher construction at each modality reduce the influence of
unreliable clients, while gradient validation prevents conflicting
external knowledge from being directly absorbed into the local model.
\item We prove that our proposed method ensures a bounded Byzantine influence
on both distillation gradients and private gradients of individual
clients after cross-modality fusion, thereby enabling stable local
optimization for honest clients under Byzantine distillation.
\item We conduct experiments under heterogeneous decentralized settings
with different model architectures, non-IID data, and multiple Byzantine
attacks. The results show that the proposed method maintains effective
collaboration among heterogeneous client models under adversarial
prediction manipulation and yields an improved local model prediction
accuracy, thanks to the multi-modality knowledge collaboration and
private-gradient validation.
\end{enumerate}
The rest of the paper is organized as follows: Section \ref{sec:Related-Work}
summarizes related work; Section \ref{sec:Problem-Statement} gives
problem formulation; Section \ref{sec:Methodology} describes the
methodology and algorithm; Section \ref{sec:Theoretical-Analysis}
provides theoretical analysis; Section \ref{sec:Experiments} presents
experimental results and Section \ref{sec:Conclusions} concludes
the paper.

\section{Related Work\label{sec:Related-Work}}

\subsection{Decentralized Federated Learning}

Decentralized federated learning (DFL) removes the centralized parameter
server and lets clients exchange information over a communication
graph to protect privacy. Representative decentralized stochastic
optimization methods, such as D-PSGD, can reduce the communication
bottleneck at the busiest node\cite{lian2017dpsgd} via mixing a client's
local model state with the states of its neighbors. Most gossip- or
consensus-based DFL methods nevertheless exchange model parameters
or gradients. They therefore assume a common architecture and an aligned
parameter space, which makes direct aggregation undefined when clients
use models with different structures and parameter sizes.

Knowledge transfer in the output space offers a natural way to relax
this requirement. FedMD showed that independently designed client
models can collaborate by matching predictions on a public dataset\cite{li2019fedmd}.
In fully decentralized networks, consensus-based multi-hop federated
distillation (CMFD) averages neighbors' predictions on shared samples
to approximate consensus in function space\cite{taya2022cmfd}. SD-Dist
further uses soft predictions on a common reference dataset to identify
similar peers and transfer knowledge among heterogeneous personalized
models\cite{ye2022sddist}. Other approaches preserve model autonomy
through a shared proxy model, as in ProxyFL\cite{kalra2023proxyfl},
or learn personalized collaboration weights from distillation-based
statistical distances, as in KD-PDFL\cite{jeong2023kdpdfl}. These
studies establish that function-space communication can support model
heterogeneity, but methods such as CMFD \cite{taya2022cmfd} and SD-Dist
\cite{ye2022sddist} mainly focus on benign knowledge collaboration
and do not consider Byzantine manipulation of exchanged predictions.

\subsection{Robust Decentralized Federated Learning}

Byzantine-robust DFL has mainly been studied in a shared parameter
space. BRIDGE applies coordinate-wise screening before decentralized
model mixing\cite{fang2022bridge}, while UBAR combines model-distance
filtering with local performance evaluation\cite{guo2022ubar}. ClippedGossip
bounds the influence of neighborhood model differences on arbitrary
communication graphs\cite{he2023clippedgossip}, and remove-then-clip
aggregation couples client filtering with clipping to improve robustness
under heterogeneous data\cite{yang2024removethenclip}. More recent
personalized DFL work allocates aggregation weights using critical
parameter indices to suppress Byzantine clients\cite{zhang2025privacybyzantine}.
Although these methods address malicious participants and, in some
cases, non-IID data, their screening rules compare coordinates, distances,
or subsets of model parameters. Such comparisons require compatible
parameter representations and may confuse architectural differences
with malicious deviations. More recently, DFL-Dual \cite{sun2024dfldual}
combines data-domain and model-domain distances with trust bootstrapping
to identify Byzantine clients, while BALANCE \cite{fang2024balance}
uses each client's local model as a similarity reference to evaluate
received client models. However, these methods still rely on comparing
model updates, model states, or model-domain distances, which are
not directly applicable when clients use heterogeneous architectures.

Robustness has also been investigated directly in the prediction space.
Cronus applies robust statistics to black-box predictions and supports
heterogeneous local models\cite{chang2019cronus}. FedTGD studies
top-$k$ and impersonation attacks against federated distillation
and filters suspicious logits using density clustering and cosine
similarity\cite{li2024fedtgd}. Roux et al. analyze the Byzantine
resilience of distillation-based FL and introduce history-aware client
weighting to strengthen robust prediction aggregation\cite{roux2025byzantinefd}.
Robust Federated Inference further studies adversarial aggregation
of predictions from heterogeneous proprietary models and robustifies
nonlinear DeepSets-based aggregators\cite{dhasade2026robustinference}.
These methods demonstrate the value of output-space defenses, but
they are primarily server-coordinated or inference-oriented and do
not directly address receiver-specific Byzantine messages during iterative
serverless distillation.

\section{Problem Statement\label{sec:Problem-Statement}}

We consider a decentralized federated learning system with $N$ clients,
denoted by $\mathcal{N}=\{1,\ldots,N\}$, communicating over a fully
connected peer-to-peer network. Each client $i$ owns a private labeled
dataset $\mathcal{D}_{i}$ and a local classifier $f_{i}(\cdot;\theta_{i}):\mathcal{X}\rightarrow\mathbb{R}^{C}$.
Clients may use different model architectures and capacities, but
all models share the same $C$-class output space. Therefore, their
model parameters need not have compatible dimensions or semantics.

All clients additionally have access to an unlabeled public dataset
$\mathcal{D}_{\mathrm{pub}}$. Its labels are not used for reliability
estimation, teacher construction, or client optimization. Let $\mathcal{A}\subset\mathcal{N}$
and $\mathcal{H}=\mathcal{N}\setminus\mathcal{A}$ denote the Byzantine
and honest client sets, respectively.

Let $\mathcal{N}_{i}=\mathcal{N}\setminus\{i\}$ denote $i$'s neighbors
and $\mathcal{A}_{i}=\mathcal{A}\cap\mathcal{N}_{i}$ its Byzantine
neighbors. During each round of collaborative distillation, while
an honest client broadcasts its genuine logits (predictions) on the
public data, a Byzantine client may send arbitrary, receiver-specific
logits to its neighbors to corrupt the distillation process.

At training round $t$, each honest client initializes its local model
as $\theta^{t,0}_{i}=\theta^{t}_{i}$ and performs $K$ local SGD
steps on its private data: 
\[
\theta^{t,k+1}_{i}=\theta^{t,k}_{i}-\eta g_{i}\left(\theta^{t,k}_{i};\xi^{t,k}_{i}\right),\quad k=0,\ldots,K-1,
\]
and obtains $\bar{\theta}^{t}_{i}=\theta^{t,K}_{i}$.

After local training steps, client $i$ performs $S$ public distillation
steps, starting from the locally trained model $\bar{\theta}^{t}_{i}$.
Let $\mathcal{B}^{t,s}_{\mathrm{pub}}$ denote the public mini-batch
used at distillation step $s$, where $s=1,\ldots,S$. Logits sent
from client $j$ are genuine and generated by $j$'s current local
model if $j$ is honest, and arbitrary messages if $j$ is Byzantine.

\section{Methodology\label{sec:Methodology}}

\subsection{Framework Overview}

The proposed method achieves Byzantine-robust decentralized FL among
heterogeneous models over $T$ training rounds. In each training round
$t$, each client $i$ first performs supervised training on its private
data to obtain the locally updated model $\bar{\theta}^{t}_{i}$.
It then sequentially processes distillation on $S$ randomly sampled
mini-batches of public data, updating the model from $\widetilde{\theta}^{t,0}_{i}=\bar{\theta}^{t}_{i}$
to $\widetilde{\theta}^{t,S}_{i}$ through the following three steps:
\begin{enumerate}
\item Logit Exchange and Standardization: Each client generates logits by
applying its local model to a randomly sampled mini-batch of shared
data and exchanges it with other clients. Since different model architectures
may produce logits with different scales, each logit vector is centered
and rescaled before comparison.
\item Robust Distillation within Each Modality: Each client compares the
received logits with a median reference across three modalities of
class prediction, boundary decision and prediction correlation. For
each modality, clients with large discrepancies are filtered out,
while the retained clients are assigned reliability weights. Their
knowledge is then aggregated to construct a modality-specific teacher,
from which the corresponding distillation gradient is obtained.
\item Cross-Modality Fusion: The reliability of each modality is first estimated
to determine its adaptive weight. A supervised gradient computed from
private data is then used to validate the three distillation gradients:
conflicting prediction and boundary gradients are removed, and conflicting
relation gradients are suppressed. Finally, the validated gradients
are adaptively weighted and combined to update the local model.
\end{enumerate}
The detailed procedure of the proposed method is presented in Algorithm
\ref{alg:MG-FD} in Section \ref{subsec:Algorithm-Description}.

\begin{algorithm*}[tbh]
\caption{Robust Multi-Modality Decentralized Federated Distillation \label{alg:MG-FD}}

\begin{algorithmic}[1]

\Require Graph $G=(V,E)$, private datasets $\{\mathcal D_i\}_{i\in V}$,
public dataset $\mathcal D_{\mathrm{pub}}$, local steps $K$, training rounds $T$,
distillation steps $S$, and learning rates $\eta,\eta_d$ and distillation strength $\{\lambda_t\}_{t=0}^{T-1}$

\State Initialize local models $\{\theta_i^0\}_{i\in V}$
\For{$t=0,1,\dots,T-1$}
    \Statex \textit{// Each client $i\in V$ performs in parallel}
    \State Perform $K$ local SGD steps on $\mathcal D_i$ to obtain $\bar{\theta}_i^t$, and set $\widetilde{\theta}_i^{t,0}=\bar{\theta}_i^t$
    \For{$s=1,2,\dots,S$}
        \Statex \textit{// Public logit exchange}
        \State Evaluate $\widetilde{\theta}_i^{t,s-1}$ on $\mathcal B_{\mathrm{pub}}^{t,s}$ to obtain $Z_i^{t,s}$
        \State Send $Z_i^{t,s}$ to clients and receive $\{Z_{j\rightarrow i}^{t,s}\}_{j\in\mathcal N_i}$
        \State Transform each $Z_{j\rightarrow i}^{t,s}$ into $\widehat Z_{j\rightarrow i}^{t,s}$ according to Eq.~\eqref{eq:logit-standardization}

        \Statex \textit{// Robust Distillation within Each Modality}
			\State Construct the median reference $M_i^{t,s}$ from $\{\widehat Z_{j\rightarrow i}^{t,s}\}_{j\in\mathcal V_i}$ according to Eq.~\eqref{eq:median-reference}
        \For{$q\in\{\mathrm{pred},\mathrm{bd},\mathrm{rel}\}$}
            \State Compare $\widehat Z_{j\rightarrow i}^{t,s}$ with $M_i^{t,s}$ at modality $q$ according to Eqs.~\eqref{eq:pred-discrepancy}, \eqref{eq:bd-discrepancy}, and \eqref{eq:rel-discrepancy}
            \State Filter clients using $d_q(\widehat Z_{j\rightarrow i}^{t,s},M_i^{t,s})$ and the adaptive threshold $\xi_i^{q,t,s}$ according to Eqs.~\eqref{eq:adaptive-threshold} and \eqref{eq:source-filtering}, obtaining $\mathcal S_i^{q,t,s}$
            \State For each $j\in\mathcal S_i^{q,t,s}$, compute $r_{ij}^{q,t,s}$ and $w_{ij}^{q,t,s}$ according to Eqs.~\eqref{eq:source-reliability} and \eqref{eq:source-weighting}
            \State Aggregate the retained client knowledge using $\{w_{ij}^{q,t,s}\}$ to construct $Q_i^{q,t,s}$
            \State Distill $Q_i^{q,t,s}$ into $\widetilde{\theta}_i^{t,s-1}$ and obtain $g_i^{q,t,s}$ according to Eqs.~\eqref{eq:pred-distillation}, \eqref{eq:bd-distillation}, and \eqref{eq:rel-distillation}
			\State Compute the reliability $s_i^{q,t,s}$ of modality $q$ according to Eq.~\eqref{eq:granularity-score}
        \EndFor

        \Statex \textit{// Cross-Modality Fusion}
		
        \State Compute the adaptive modality weights $\{\omega_i^{q,t,s}\}_{q\in\mathcal Q}$ according to Eq.~\eqref{eq:granularity-weight}
        \State Sample a private labeled mini-batch and compute $g_i^{\mathrm{sup},t,s}$ according to Eq.~\eqref{eq:supervised-gradient}
        \State Validate the class prediction and boundary decision gradients according to  Eq.~\eqref{eq:gradient-projection}, and validate the prediction correlation gradient according to Eq.~\eqref{eq:relation-validation}
        \State Combine the validated gradients according to Eq.~\eqref{eq:combined-gradient}
        \State Update $\widetilde{\theta}_i^{t,s}$ according to Eq.~\eqref{eq:public-update}
    \EndFor
    \State Set $\theta_i^{t+1}=\widetilde{\theta}_i^{t,S}$
\EndFor
\State \textbf{Output:} Local models $\{\theta_i^T\}_{i\in V}$
\end{algorithmic}
\end{algorithm*}

\subsection{Logit Exchange and Standardization \label{subsec:Logit-Exchange-and}}

Let $\mathcal{B}^{t,s}_{\mathrm{pub}}=\{x^{t,s}_{b}\}^{B}_{b=1}$
denote the $s$-th mini-batch of public data used at distillation
step $s$ of training round $t$. On the $s$-th mini-batch, honest
client $j$ evaluates its current model $\widetilde{\theta}^{t,s-1}_{j}$
and obtains 
\[
Z^{t,s}_{j}=f_{j}\left(\mathcal{B}^{t,s}_{\mathrm{pub}};\widetilde{\theta}^{t,s-1}_{j}\right)\in\mathbb{R}^{B\times C}.
\]
For receiver $i$, the logit attributed to client $j$ is denoted
by $Z^{t,s}_{j\rightarrow i}$. An honest client sends its genuine
logits, 
\[
Z^{t,s}_{j\rightarrow i}=Z^{t,s}_{j},
\]
whereas a Byzantine client may send an arbitrary finite logits and
may transmit different messages to different receivers.

Because heterogeneous model architectures may produce logits with
different offsets and scales, each logit is centered by subtracting
its mean and rescaled by its standard deviation across the $C$ classes.
The resulting logits have approximately zero mean and scale at most
one, making outputs from heterogeneous models directly comparable
while preserving their relative class ordering. For $z\in\mathbb{R}^{C}$,
define 
\begin{equation}
\mathcal{S}(z)=\frac{z-\mu(z)\mathbf{1}_{C}}{\sigma(z)+\epsilon},\label{eq:logit-standardization}
\end{equation}
where $\mu(z)=\frac{1}{C}\sum^{C}_{c=1}z_{c}$ and $\sigma(z)=\left[\frac{1}{C}\sum^{C}_{c=1}\left(z_{c}-\mu(z)\right)^{2}\right]^{1/2}$,
and $\epsilon>0$ prevents division by zero. This operation is applied
row-wise to obtain $\widehat{Z}^{t,s}_{j\rightarrow i}=\mathcal{S}(Z^{t,s}_{j\rightarrow i})$.

\subsection{Robust Distillation within Each Modality}

Let $\mathcal{V}_{i}=\mathcal{N}_{i}\cup\{i\}$ denote client $i$
and its neighbors. Client $i$ first constructs a coordinate-wise
median reference 
\begin{equation}
M_{i}=CoordMedian_{j\in\mathcal{V}_{i}}\widehat{Z}_{j\rightarrow i}.\label{eq:median-reference}
\end{equation}
The median is computed independently for each public sample and class
coordinate, and is used only to evaluate the received logits.

Based on this reference $M_{i}$, client $i$ then compares the logits
from each client across three modalities: class prediction, boundary
decision, and prediction correlation.

\subsubsection{Multi-Modality Logit Comparison \label{subsubsec:source-evaluation}}

For the three modalities 
\[
\mathcal{Q}=\{\mathrm{pred},\mathrm{bd},\mathrm{rel}\},
\]
client $i$ measures the difference between each client's standardized
logits $\widehat{Z}_{j\rightarrow i}$ and the median reference $M_{i}$.

For class prediction knowledge, we define 
\begin{equation}
d_{\mathrm{pred}}\left(\widehat{Z}_{j\rightarrow i},M_{i}\right)=\frac{1}{B}\sum^{B}_{b=1}\left(1-\frac{\left\langle \widehat{Z}_{j\rightarrow i,b},M_{i,b}\right\rangle }{\left\Vert \widehat{Z}_{j\rightarrow i,b}\right\Vert _{2}\left\Vert M_{i,b}\right\Vert _{2}+\epsilon}\right).\label{eq:pred-discrepancy}
\end{equation}
This discrepancy measures the difference in class preferences between
client $j$ and the median reference.

For boundary decision knowledge on logit matrix $A$, letting $A_{b,(1)}$
and $A_{b,(2)}$ denote the largest and second-largest entries of
row $A_{b}$, respectively, we define 
\[
m_{b}(A)=A_{b,(1)}-A_{b,(2)},\qquad P_{b}(A)=Softmax(A_{b}/T_{\mathrm{kd}}),
\]
where $T_{\mathrm{kd}}>0$ denotes the distillation temperature used
to smooth the class probabilities, with a larger $T_{\mathrm{kd}}$
producing a smoother distribution over classes, and 
\[
\Phi_{\mathrm{bd}}(A)_{b}=\left[m_{b}(A);P_{b}(A)\right].
\]
The boundary discrepancy is then 
\begin{equation}
d_{\mathrm{bd}}\left(\widehat{Z}_{j\rightarrow i},M_{i}\right)=\frac{\left\Vert \Phi_{\mathrm{bd}}\left(\widehat{Z}_{j\rightarrow i}\right)-\Phi_{\mathrm{bd}}(M_{i})\right\Vert _{1}}{B(C+1)}.\label{eq:bd-discrepancy}
\end{equation}
This discrepancy measures the difference in boundary decision and
class probabilities.

For prediction correlation knowledge on logit matrix $A$, we define
\[
H_{b}(A)=\frac{A_{b}}{\|A_{b}\|_{2}+\epsilon},\qquad R(A)=H(A)H(A)^{\top}.
\]
The relation discrepancy is defined by 
\begin{equation}
d_{\mathrm{rel}}\left(\widehat{Z}_{j\rightarrow i},M_{i}\right)=\frac{\left[\sum_{b\neq b'}\left(R(\widehat{Z}_{j\rightarrow i})_{b,b'}-R(M_{i})_{b,b'}\right)^{2}\right]^{1/2}}{\max\left\{ \left[\sum_{b\neq b'}R(M_{i})^{2}_{b,b'}\right]^{1/2},\epsilon\right\} }.\label{eq:rel-discrepancy}
\end{equation}
This discrepancy measures the difference in pairwise relations among
different public samples.

Therefore, each client is evaluated against the same median reference
from class prediction, boundary decision, and prediction correlation
perspectives.

\subsubsection{Robust Neighbor Filtering and Weighting \label{subsubsec:source-filtering}}

For each modality, client $i$ first filters out clients whose logits
deviate too much from the median reference, and then assigns reliability
weights to the remaining clients before aggregation.

Client $i$ first determines an adaptive threshold to identify clients
whose logits deviate excessively from the median reference. Specifically,
\[
c^{q}_{i}=Median_{j\in\mathcal{V}_{i}\setminus\{i\}}d_{q}\left(\widehat{Z}_{j\rightarrow i},M_{i}\right).
\]
The corresponding median absolute deviation is 
\[
MAD^{q}_{i}=Median_{j\in\mathcal{V}_{i}\setminus\{i\}}\left|d_{q}\left(\widehat{Z}_{j\rightarrow i},M_{i}\right)-c^{q}_{i}\right|.
\]
The filtering threshold is 
\begin{equation}
\xi^{q}_{i}=c^{q}_{i}+\kappa\left(MAD^{q}_{i}+\epsilon_{\mathrm{MAD}}\right),\label{eq:adaptive-threshold}
\end{equation}
where $\kappa$ controls the filtering tolerance and $\epsilon_{\mathrm{MAD}}>0$
prevents a degenerate threshold when MAD is zero. Using the median
and MAD makes the threshold less sensitive to extreme Byzantine logits
and allows it to adapt to the current discrepancy distribution at
each modality. Client $i$ retains itself and all clients whose discrepancies
are below the threshold: 
\begin{equation}
\mathcal{S}^{q}_{i}=\{i\}\cup\left\{ j\in\mathcal{N}_{i}:d_{q}\left(\widehat{Z}_{j\rightarrow i},M_{i}\right)\le\xi^{q}_{i}\right\} .\label{eq:source-filtering}
\end{equation}

Filtering removes clearly unreliable clients, but some Byzantine clients
may still remain. Therefore, client $i$ further assigns a reliability
score to each retained client: 
\begin{equation}
r^{q}_{ij}=\alpha_{\mathrm{med}}d_{q}\left(\widehat{Z}_{j\rightarrow i},M_{i}\right)+\beta_{\mathrm{local}}d_{q}\left(\widehat{Z}_{j\rightarrow i},\widehat{Z}_{i}\right),\qquad j\in\mathcal{S}^{q}_{i}.\label{eq:source-reliability}
\end{equation}
Here, the first term measures the difference from the median reference,
and the second measures the difference from client $i$'s own prediction.
The retained clients are then weighted as 
\begin{equation}
w^{q}_{ij}=\frac{\exp(-r^{q}_{ij}/T_{w})}{\sum_{\ell\in\mathcal{S}^{q}_{i}}\exp(-r^{q}_{i\ell}/T_{w})}.\label{eq:source-weighting}
\end{equation}
Thus, more reliable clients receive larger weights in the subsequent
aggregation.

\subsubsection{Modality-Specific Teacher Construction and Distillation \label{subsubsec:granularity-distillation}}

After filtering and weighting, client $i$ constructs one teacher
for each modality and distills the corresponding knowledge into its
local model.

At distillation step $s$ of training round $t$, the current logit
matrix $Z^{t,s}_{i}$ defined in Section \ref{subsec:Logit-Exchange-and}
is used as the student output, i.e., 
\[
Z^{\mathrm{stu},t,s}_{i}=Z^{t,s}_{i}.
\]
Its standardized logits and normalized probabilities are 
\[
\widehat{Z}^{\mathrm{stu},t,s}_{i}=\mathcal{S}\left(Z^{\mathrm{stu},t,s}_{i}\right),
\]
\[
P^{\mathrm{stu},t,s}_{i}=Softmax\left(\widehat{Z}^{\mathrm{stu},t,s}_{i}/T_{\mathrm{kd}}\right).
\]

\paragraph{Class Prediction Knowledge}

The class prediction teacher is obtained by weighted aggregation:
\[
Q^{\mathrm{pred}}_{i}=\mathcal{S}\left(\sum_{j\in\mathcal{S}^{\mathrm{pred}}_{i}}w^{\mathrm{pred}}_{ij}\widehat{Z}_{j\rightarrow i}\right),
\]
\[
P^{\mathrm{pred}}_{i}=Softmax\left(Q^{\mathrm{pred}}_{i}/T_{\mathrm{kd}}\right).
\]
For each public sample $x_{b}$, define its teacher label and weighted
voting confidence as 
\[
\widehat{y}^{\mathrm{pred}}_{i,b}=\arg\max_{c}Q^{\mathrm{pred}}_{i,b,c},
\]
\[
v_{i,b}=\sum_{j\in\mathcal{S}^{\mathrm{pred}}_{i}}w^{\mathrm{pred}}_{ij}\mathbf{1}\left[\arg\max_{c}\widehat{Z}_{j\rightarrow i,b,c}=\widehat{y}^{\mathrm{pred}}_{i,b}\right].
\]
Let $a_{i,b}=v_{i,b}\mathbf{1}[v_{i,b}\ge\tau_{\mathrm{conf}}],$
where $\tau_{\mathrm{conf}}\in[0,1]$ is the confidence threshold
for retaining the teacher label. The class prediction distillation
loss is 
\begin{align}
\mathcal{L}^{\mathrm{pred},t,s}_{i} & =\frac{T^{2}_{\mathrm{kd}}}{B}\sum^{B}_{b=1}D_{\mathrm{KL}}\left(P^{\mathrm{pred}}_{i,b}\|P^{\mathrm{stu},t,s}_{i,b}\right)\label{eq:pred-distillation}\\
 & \quad+\mu\frac{\sum^{B}_{b=1}a_{i,b}CE\left(Z^{\mathrm{stu},t,s}_{i,b},\widehat{y}^{\mathrm{pred}}_{i,b}\right)}{\sum^{B}_{b=1}a_{i,b}+\epsilon}.\nonumber 
\end{align}
The corresponding distillation gradient is 
\[
g^{\mathrm{pred},t,s}_{i}=\nabla_{\theta_{i}}\mathcal{L}^{\mathrm{pred},t,s}_{i}.
\]

\paragraph{Boundary Decision Knowledge.}

Similarly, the boundary decision teacher is 
\[
Q^{\mathrm{bd}}_{i}=\mathcal{S}\left(\sum_{j\in\mathcal{S}^{\mathrm{bd}}_{i}}w^{\mathrm{bd}}_{ij}\widehat{Z}_{j\rightarrow i}\right),
\]
\[
P^{\mathrm{bd}}_{i}=Softmax\left(Q^{\mathrm{bd}}_{i}/T_{\mathrm{kd}}\right).
\]
For sample $x_{b}$, let $\widehat{y}^{\mathrm{bd}}_{i,b}=\arg\max_{c}Q^{\mathrm{bd}}_{i,b,c}.$
The teacher and student margins are 
\[
m^{\mathrm{bd}}_{i,b}=Q^{\mathrm{bd}}_{i,b,\widehat{y}^{\mathrm{bd}}_{i,b}}-\max_{c\neq\widehat{y}^{\mathrm{bd}}_{i,b}}Q^{\mathrm{bd}}_{i,b,c},
\]
\[
m^{\mathrm{stu}}_{i,b}=\widehat{Z}^{\mathrm{stu},t,s}_{i,b,\widehat{y}^{\mathrm{bd}}_{i,b}}-\max_{c\neq\widehat{y}^{\mathrm{bd}}_{i,b}}\widehat{Z}^{\mathrm{stu},t,s}_{i,b,c}.
\]
The boundary decision distillation loss is 
\begin{align}
\mathcal{L}^{\mathrm{bd},t,s}_{i} & =\frac{1}{B}\sum^{B}_{b=1}\left|m^{\mathrm{stu}}_{i,b}-m^{\mathrm{bd}}_{i,b}\right|\label{eq:bd-distillation}\\
 & \quad+\frac{\zeta T^{2}_{\mathrm{kd}}}{B}\sum^{B}_{b=1}D_{\mathrm{KL}}\left(P^{\mathrm{bd}}_{i,b}\|P^{\mathrm{stu},t,s}_{i,b}\right).\nonumber 
\end{align}
The corresponding gradient is 
\[
g^{\mathrm{bd},t,s}_{i}=\nabla_{\theta_{i}}\mathcal{L}^{\mathrm{bd},t,s}_{i}.
\]

\paragraph{Prediction Correlation Knowledge.}

For inter-sample prediction correlation knowledge, client $i$ directly
aggregates the relation matrices of the retained clients and takes
this as its teacher: 
\[
Q^{\mathrm{rel}}_{i}=\sum_{j\in\mathcal{S}^{\mathrm{rel}}_{i}}w^{\mathrm{rel}}_{ij}R\left(\widehat{Z}_{j\rightarrow i}\right).
\]
The student relation matrix is 
\[
R^{\mathrm{stu},t,s}_{i}=R\left(\widehat{Z}^{\mathrm{stu},t,s}_{i}\right).
\]
The prediction correlation distillation loss is 
\begin{equation}
\mathcal{L}^{\mathrm{rel},t,s}_{i}=\frac{1}{B(B-1)}\sum_{b\neq b'}\left(R^{\mathrm{stu},t,s}_{i,b,b'}-Q^{\mathrm{rel}}_{i,b,b'}\right)^{2},\label{eq:rel-distillation}
\end{equation}
and relation distillation gradient is 
\[
g^{\mathrm{rel},t,s}_{i}=\nabla_{\theta_{i}}\mathcal{L}^{\mathrm{rel},t,s}_{i}.
\]
Therefore, the three modalities produce three distillation gradients,
$g^{\mathrm{pred},t,s}_{i}$, $g^{\mathrm{bd},t,s}_{i}$, and $g^{\mathrm{rel},t,s}_{i}$,
which are subsequently validated and combined.

\subsection{Cross-Modality Fusion}

\subsubsection{Adaptive Modality Weighting \label{subsubsec:granularity-weighting}}

The three knowledge modalities may have different reliability at different
clients and distillation steps. We therefore estimate the overall
reliability of modality $q$ from its retained clients as 
\begin{equation}
s^{q}_{i}=\sum_{j\in\mathcal{S}^{q}_{i}}w^{q}_{ij}\exp(-r^{q}_{ij}).\label{eq:granularity-score}
\end{equation}
A larger $s^{q}_{i}$ indicates that the retained clients at modality
$q$ have smaller reliability costs and provide more consistent knowledge.

Let $\boldsymbol{\omega}^{0}=\left(\omega^{0}_{\mathrm{pred}},\omega^{0}_{\mathrm{bd}},\omega^{0}_{\mathrm{rel}}\right)$,
and $\sum_{q\in\mathcal{Q}}\omega^{0}_{q}=1,$ denote the initial
modality weights, and define 
\[
\bar{s}_{i}=\frac{1}{3}\sum_{q\in\mathcal{Q}}s^{q}_{i}.
\]
The adaptive modality weight is 
\begin{equation}
\omega^{q}_{i}=\frac{\exp\left[\log\omega^{0}_{q}+\gamma_{\omega}(s^{q}_{i}-\bar{s}_{i})/T_{\omega}\right]}{{\displaystyle \sum_{\ell\in\mathcal{Q}}\exp\left[\log\omega^{0}_{\ell}+\gamma_{\omega}(s^{\ell}_{i}-\bar{s}_{i})/T_{\omega}\right]}}.\label{eq:granularity-weight}
\end{equation}
Thus, 
\[
\omega^{q}_{i}\ge0,\qquad\sum_{q\in\mathcal{Q}}\omega^{q}_{i}=1.
\]

A modality whose retained clients are more reliable than the current
average receives a larger weight in the final public update. It is
worth distinguishing the two levels of weighting used in the method.The
client weight $w^{q}_{ij}$ controls the contribution of client $j$
within modality $q$, whereas $\omega^{q}_{i}$ controls the contribution
of modality $q$ in the final multi-modality update.

\subsubsection{Private-Gradient Validation and Model Update \label{subsubsec:private-validation}}

At distillation step $s$ of training round $t$, honest client $i$
samples a private labeled mini-batch $\mathcal{B}^{\mathrm{sup},t,s}_{i}\subset\mathcal{D}_{i}$
and computes the supervised gradient at the same parameter point used
for public distillation: 
\begin{equation}
g^{\mathrm{sup},t,s}_{i}=\nabla_{\theta_{i}}\mathcal{L}_{\mathrm{CE}}\left(\mathcal{B}^{\mathrm{sup},t,s}_{i};\widetilde{\theta}^{t,s-1}_{i}\right).\label{eq:supervised-gradient}
\end{equation}
This private gradient is used only to validate the external distillation
directions and does not introduce an additional supervised update
during the public phase.

For $q\in\{\mathrm{pred},\mathrm{bd}\}$, if a distillation gradient
conflicts with the private supervised gradient, its conflicting component
is removed:

{\footnotesize
\begin{equation}
\widetilde{g}^{q,t,s}_{i}=\begin{cases}
g^{q,t,s}_{i}-\dfrac{\left\langle g^{\mathrm{sup},t,s}_{i},g^{q,t,s}_{i}\right\rangle }{\left\Vert g^{\mathrm{sup},t,s}_{i}\right\Vert ^{2}_{2}}g^{\mathrm{sup},t,s}_{i}, & \begin{array}{l}
\left\langle g^{\mathrm{sup},t,s}_{i},g^{q,t,s}_{i}\right\rangle <0,\\
\left\Vert g^{\mathrm{sup},t,s}_{i}\right\Vert _{2}>0,
\end{array}\\[14pt]
g^{q,t,s}_{i}, & \text{otherwise},
\end{cases}.\label{eq:gradient-projection}
\end{equation}
}The projected gradients therefore satisfy 
\[
\left\langle g^{\mathrm{sup},t,s}_{i},\widetilde{g}^{q,t,s}_{i}\right\rangle \ge0,\qquad q\in\{\mathrm{pred},\mathrm{bd}\}.
\]
Prediction correlation knowledge is treated more conservatively. Its
gradient is retained only when it agrees with the private supervised
direction: 
\begin{equation}
\widetilde{g}^{\mathrm{rel},t,s}_{i}=\mathbf{1}\left[\left\langle g^{\mathrm{sup},t,s}_{i},g^{\mathrm{rel},t,s}_{i}\right\rangle >0\right]g^{\mathrm{rel},t,s}_{i}.\label{eq:relation-validation}
\end{equation}

At distillation step $s$ of training round $t$, the validated multi-modality
gradient is 
\begin{equation}
g^{\mathrm{MG},t,s}_{i}=\sum_{q\in\mathcal{Q}}\omega^{q,t,s}_{i}\widetilde{g}^{q,t,s}_{i}.\label{eq:combined-gradient}
\end{equation}
Since all modality weights are nonnegative and every validated modality-wise
gradient has a nonnegative inner product with the private supervised
gradient, 
\[
\left\langle g^{\mathrm{sup},t,s}_{i},g^{\mathrm{MG},t,s}_{i}\right\rangle \ge0.
\]
Client $i$ then performs the public update 
\begin{equation}
\widetilde{\theta}^{t,s}_{i}=\widetilde{\theta}^{t,s-1}_{i}-\eta_{d}\lambda_{t}g^{\mathrm{MG},t,s}_{i},\label{eq:public-update}
\end{equation}
where $\eta_{d}>0$ is the distillation learning rate and $\lambda_{t}\ge0$
controls the overall strength of public knowledge transfer.

After processing all $S$ public mini-batches, 
\[
\theta^{t+1}_{i}=\widetilde{\theta}^{t,S}_{i}.
\]
The resulting procedure first controls unreliable knowledge within
each modality through client filtering and reliability weighting,
and then controls the interaction among the resulting distillation
gradients through private-gradient validation and adaptive modality
fusion.

\subsection{Algorithm Description \label{subsec:Algorithm-Description}}

The complete procedure implementing the proposed method is summarized
in Algorithm \ref{alg:MG-FD}. Specifically, each client first performs
local supervised training and then exchanges logits with its neighbors
on the shared public data. The received logits are transformed to
reduce architecture-dependent offset and scale differences. For each
modality, the client compares the received logits with the median
reference, filters unreliable clients, assigns reliability weights
to the retained clients, and constructs a teacher to obtain the corresponding
distillation gradient. Finally, the three distillation gradients are
adaptively weighted according to their modality reliability, validated
using a private supervised gradient, and combined to update the local
model.

\section{Theoretical Analysis\label{sec:Theoretical-Analysis}}

We analyze the optimization behavior of each honest client separately.
Since heterogeneous clients may use different model architectures
and therefore have incompatible parameter spaces, convergence to a
common model parameter is not well-defined. Instead, we study whether
each honest client's own supervised objective remains stable under
Byzantine public distillation. Specifically, we bound the residual
Byzantine influence on the multi-modality distillation gradient, analyze
the safety of private-gradient validation, and then establish bounded
average stationarity for each honest client.

Within training round $t$, honest client $i$ first performs $K$
local supervised SGD steps from $\theta^{t,0}_{i}=\theta^{t}_{i}$
and obtains $\bar{\theta}^{t}_{i}=\theta^{t,K}_{i}$. The public distillation
phase then starts from $\widetilde{\theta}^{t,0}_{i}=\bar{\theta}^{t}_{i}$
and processes $S$ public mini-batches sequentially. At distillation
step $s$, the class prediction, boundary decision, and prediction
correlation distillation gradients are computed at $\widetilde{\theta}^{t,s-1}_{i}$
and validated using a private supervised gradient, producing the validated
multi-modality gradient $g^{\mathrm{MG},t,s}_{i}$. The public update
is 
\begin{equation}
\widetilde{\theta}^{t,s}_{i}=\widetilde{\theta}^{t,s-1}_{i}-\eta_{d}\lambda_{t}g^{\mathrm{MG},t,s}_{i}.
\end{equation}
After $S$ distillation steps, $\theta^{t+1}_{i}=\widetilde{\theta}^{t,S}_{i}$.
The private supervised gradient is used only for validating the public
distillation directions and does not introduce an additional model
update.

\subsection{Assumptions}

For each honest client $i$, let $F_{i}(\theta)$ denote its local
supervised objective. We make the following assumptions for the convergence
analysis.

\subsubsection*{Assumption 1 (Smooth and lower-bounded local objective)}

For every honest client $i$, $F_{i}$ is $L$-smooth and lower bounded.
Specifically, for any $\theta$ and $\theta'$, 
\begin{equation}
F_{i}(\theta')\leq F_{i}(\theta)+\left\langle \nabla F_{i}(\theta),\theta'-\theta\right\rangle +\frac{L}{2}\left\Vert \theta'-\theta\right\Vert ^{2}_{2},
\end{equation}
and 
\begin{equation}
F_{i}(\theta)\geq F^{\inf}_{i},
\end{equation}
where $F^{\inf}_{i}$ is a finite lower bound of $F_{i}$.

\subsubsection*{Assumption 2 (Unbiased stochastic gradients with bounded second moment)}

For every honest client $i$, let $g_{i}(\theta;\xi)$ denote a stochastic
gradient computed from private data at model parameter $\theta$.
We assume 
\begin{equation}
\mathbb{E}_{\xi}\left[g_{i}(\theta;\xi)\right]=\nabla F_{i}(\theta),
\end{equation}
and that there exists a finite constant $G_{g}>0$ such that 
\begin{equation}
\mathbb{E}_{\xi}\left[\left\Vert g_{i}(\theta;\xi)\right\Vert ^{2}_{2}\right]\leq G^{2}_{g}.
\end{equation}

\subsubsection*{Assumption 3 (Bounded sensitivity of distillation gradients to teacher
representations).}

For each modality $q\in\{\mathrm{pred},\mathrm{bd},\mathrm{rel}\}$,
let $\Gamma^{t,s}_{i,q}(Q)$ denote the distillation gradient of honest
client $i$ induced by the teacher representation $Q$, which is obtained
after client selection, reliability weighting, and teacher construction
on the $s$-th mini-batch of the $t$-th training round. There exists
a finite constant $L_{q}>0$ such that
\begin{equation}
\left\Vert \Gamma^{t,s}_{i,q}(Q)-\Gamma^{t,s}_{i,q}(Q')\right\Vert _{2}\leq L_{q}\left\Vert Q-Q'\right\Vert _{F}.
\end{equation}

This condition states that a bounded change in the teacher representation
induces a bounded change in the corresponding distillation gradient.

\subsubsection*{Assumption 4 (Bounded honest distillation gradients).}

Along the optimization trajectory of every honest client $i$, there
exists a finite constant $G^{H}_{i}>0$ such that 
\begin{equation}
\sum_{q\in\{\mathrm{pred},\mathrm{bd},\mathrm{rel}\}}\omega^{q,t,s}_{i}\left\Vert g^{q,t,s}_{i,H}\right\Vert _{2}\leq G^{H}_{i},\qquad\forall t,s,
\end{equation}
where $g^{q,t,s}_{i,H}$ denotes the distillation gradient induced
by the honest teacher at modality $q$.

This condition bounds the weighted magnitude of the honest distillation
gradients along the optimization trajectory.

\subsection{Main Result}

We first state the main result and then establish the supporting properties
in the subsequent lemmas.

Let $\gamma=\eta K$, and $\alpha=\sup_{t}\eta_{d}\lambda_{t}$, where
$\gamma$ is the local-update step size and $\alpha$ is an upper
bound on the public-distillation step size. Let $\delta_{i}$ denote
the error bound of the private validation gradient and $G_{i}$ denote
the bound on the validated multi-modality distillation gradient, which
will be established below.
\begin{thm}
(Local convergence under Byzantine distillation) \label{thm:main-convergence}
Suppose Assumptions 1-{}-4 hold. Consider any honest client $i$ that
follows Algorithm $\ref{alg:MG-FD}$ for $T$ training rounds, where
each round consists of $K$ local supervised SGD steps followed by
$S$ validated public distillation steps. Let $\gamma=\eta K$ and
$\alpha=\sup_{t}\eta_{d}\lambda_{t}$. If $0<\gamma\leq\frac{1}{4L}$,
then the sequence of local models $\{\theta^{t}_{i}\}^{T}_{t=0}$
generated by Algorithm $\ref{alg:MG-FD}$ satisfies 
\begin{align*}
\frac{1}{T}\sum^{T-1}_{t=0}\mathbb{E}\left[\left\Vert \nabla F_{i}(\theta^{t}_{i})\right\Vert ^{2}_{2}\right] & \leq\frac{2\left(F_{i}(\theta^{0}_{i})-F^{\inf}_{i}\right)}{\gamma T}\\
 & \quad+\frac{5}{8}L^{2}G^{2}_{g}\eta^{2}K^{2}+L\gamma G^{2}_{g}.\\
 & \quad+\frac{2S\alpha\delta_{i}G_{i}}{\gamma}+\frac{LS\alpha^{2}G^{2}_{i}}{\gamma}
\end{align*}
\end{thm}
The theorem shows that, for every honest client, the average squared
gradient norm of its local supervised objective remains bounded under
Byzantine public distillation. The first term decreases as $\mathcal{O}(1/T)$,
while the remaining terms characterize the residual effects of local-update
drift, stochastic gradients, private-gradient validation, and public
distillation.
\begin{cor}
Under the conditions of the theorem, choose $\gamma=T^{-1/2}$ and
$\alpha\leq T^{-1}$, equivalently $\eta=1/(K\sqrt{T})$ and $\eta_{d}\lambda_{t}\leq1/T$,
with $T\geq16L^{2}$. Then 
\[
\frac{1}{T}\sum^{T-1}_{t=0}\mathbb{E}\left[\|\nabla F_{i}(\theta^{t}_{i})\|^{2}_{2}\right]=O(T^{-1/2}).
\]
 Thus a uniformly sampled local iterate converges to stationarity
in expectation.
\end{cor}
\begin{IEEEproof}
Substituting $\gamma=T^{-1/2}$, $\eta K=T^{-1/2}$, and $\alpha\leq T^{-1}$
into the theorem gives orders $T^{-1/2}$, $T^{-1}$, $T^{-1/2}$,
$T^{-1/2}$, and $T^{-3/2}$ for its five terms, respectively. The
step-size condition follows from $T^{-1/2}\leq1/(4L)$.
\end{IEEEproof}

\subsubsection*{Proof sketch}

To prove Theorem $\ref{thm:main-convergence}$, we analyze how the
public distillation stage affects the descent provided by local supervised
training. The proof proceeds in three steps:
\begin{enumerate}
\item We first derive the decrease of the local supervised objective after
the $K$ local SGD updates. This part follows the same analysis as
our previous work \cite{ma2026rdpfl}.
\item We then analyze the Byzantine knowledge after filtering and reliability
weighting within each knowledge modality. Byzantine clients are not
necessarily completely removed, but their remaining contribution to
each distillation gradient is bounded by the total weight assigned
to the retained Byzantine clients. Consequently, the magnitude of
the Byzantine distillation gradients is bounded by $G_{i}$, as established
in Lemma $\ref{lem:byzantine-distillation-bound}$.
\item We next analyze the cross-modality fusion process. The adaptive modality
weights determine the contribution of the three distillation gradients,
while private-gradient validation projects or suppresses gradients
that conflict with the local supervised direction. After validation
and fusion, the final multi-modality gradient remains bounded by $G_{i}$,
and its possible adverse effect on the local objective is further
bounded by $\delta_{i}G_{i}$, as shown in Lemma $\ref{lem:safe-private-validation}$.
\end{enumerate}
Combining the local descent with these two bounds shows that Byzantine
public distillation introduces only bounded perturbation terms into
the local optimization. Therefore, the descent achieved by local SGD
cannot be unboundedly disrupted by the Byzantine knowledge, which
leads to the convergence bound in Theorem $\ref{thm:main-convergence}.$

\subsection{Bounded Residual Byzantine Distillation}

The public-distillation terms in Theorem $\ref{thm:main-convergence}$
depend on $G_{i}$, which bounds the multi-modality distillation gradients.
We next show that $G_{i}$ consists of a bounded honest component
and a residual Byzantine component determined by the weights assigned
to retained Byzantine sources.

Let $\mathcal{Q}=\{\mathrm{pred},\mathrm{bd},\mathrm{rel}\}$. For
each $q\in\mathcal{Q}$, let $\mathcal{S}^{q,t,s}_{i}$ and $w^{q,t,s}_{ij}$
denote the retained source set and the corresponding normalized weights.
Define the residual Byzantine weight as 
\[
\beta^{q,t,s}_{i}=\sum_{j\in\mathcal{A}_{i}\cap\mathcal{S}^{q,t,s}_{i}}w^{q,t,s}_{ij}.
\]
Since the trusted local source of an honest receiver is always retained,
\[
0\leq\beta^{q,t,s}_{i}<1.
\]

Let $Q^{q,t,s}_{j\rightarrow i}$ denote the source representation
at modality $q$, and define 
\[
Q^{q,t,s}_{i}=\sum_{j\in\mathcal{S}^{q,t,s}_{i}}w^{q,t,s}_{ij}Q^{q,t,s}_{j\rightarrow i},
\]
and its normalized honest component as 
\[
Q^{q,t,s}_{i,H}=\frac{1}{1-\beta^{q,t,s}_{i}}\sum_{j\in\mathcal{H}_{i}\cap\mathcal{S}^{q,t,s}_{i}}w^{q,t,s}_{ij}Q^{q,t,s}_{j\rightarrow i}.
\]

The corresponding distillation gradients are 
\[
g^{q,t,s}_{i}=\Gamma^{t,s}_{i,q}(Q^{q,t,s}_{i}),\qquad g^{q,t,s}_{i,H}=\Gamma^{t,s}_{i,q}(Q^{q,t,s}_{i,H}).
\]

\begin{lem}
(Bounded residual Byzantine distillation). \label{lem:byzantine-distillation-bound}
Suppose Assumptions\textasciitilde 3-{}-4 hold. For every honest
client $i$, modality $q\in\mathcal{Q}$, training round $t$, and
distillation step $s$, the difference between the actual distillation
gradient and the gradient induced by the retained honest sources satisfies
\begin{equation}
\left\Vert g^{q,t,s}_{i}-g^{q,t,s}_{i,H}\right\Vert _{2}\leq2L_{q}M_{q}\beta^{q,t,s}_{i},\label{eq:byz-gradient-perturbation}
\end{equation}
where $M_{q}=\begin{cases}
\sqrt{BC}, & q\in\{\mathrm{pred},\mathrm{bd}\}\\
B, & q=\mathrm{rel}.
\end{cases}$. Consequently, 
\begin{equation}
\sum_{q\in\mathcal{Q}}\omega^{q,t,s}_{i}\left\Vert g^{q,t,s}_{i}\right\Vert _{2}\leq G_{i},\label{eq:distillation-gradient-bound}
\end{equation}
where $G_{i}=G^{H}_{i}+\varepsilon^{\mathrm{Byz}}_{i}$ and $\varepsilon^{\mathrm{Byz}}_{i}=\sup_{t,s}\sum_{q\in\mathcal{Q}}2\omega^{q,t,s}_{i}L_{q}M_{q}\beta^{q,t,s}_{i}$.
\end{lem}
\begin{IEEEproof}
We first note that the normalized weights are positive and sum to
one over $\mathcal{S}^{q,t,s}_{i}$. Since the trusted local source
$i$ of an honest receiver is always retained, 
\[
\beta^{q,t,s}_{i}=1-\sum_{j\in\mathcal{H}_{i}\cap\mathcal{S}^{q,t,s}_{i}}w^{q,t,s}_{ij}\leq1-w^{q,t,s}_{ii}<1.
\]
Hence, $0\leq\beta^{q,t,s}_{i}<1$.

We next bound the representation contributed by each retained source.
For a standardized logit vector $\widehat{z}\in\mathbb{R}^{C}$, 
\[
\left\Vert \widehat{z}\right\Vert ^{2}_{2}=\frac{\sum^{C}_{c=1}(z_{c}-\mu(z))^{2}}{(\sigma(z)+\epsilon)^{2}}=\frac{C\sigma^{2}(z)}{(\sigma(z)+\epsilon)^{2}}\leq C.
\]
Therefore, for a public mini-batch of size $B$, 
\[
\left\Vert \widehat{Z}^{t,s}_{j\rightarrow i}\right\Vert _{F}\leq\sqrt{BC}.
\]
For the prediction correlation modality, each entry of $R(\widehat{Z}^{t,s}_{j\rightarrow i})$
is the inner product of two row-normalized vectors and therefore has
magnitude at most one. Since the relation matrix has size $B\times B$,
\[
\left\Vert R\left(\widehat{Z}^{t,s}_{j\rightarrow i}\right)\right\Vert _{F}\leq B.
\]
Thus, the corresponding source representations satisfy 
\begin{equation}
\left\Vert Q^{q,t,s}_{j\rightarrow i}\right\Vert _{F}\leq M_{q}.\label{eq:source-representation-bound}
\end{equation}
For $\beta^{q,t,s}_{i}>0$, define the normalized Byzantine component
as 
\[
Q^{q,t,s}_{i,A}=\frac{1}{\beta^{q,t,s}_{i}}\sum_{j\in\mathcal{A}_{i}\cap\mathcal{S}^{q,t,s}_{i}}w^{q,t,s}_{ij}Q^{q,t,s}_{j\rightarrow i},
\]
and the normalized honest component as 
\[
Q^{q,t,s}_{i,H}=\frac{1}{1-\beta^{q,t,s}_{i}}\sum_{j\in\mathcal{H}_{i}\cap\mathcal{S}^{q,t,s}_{i}}w^{q,t,s}_{ij}Q^{q,t,s}_{j\rightarrow i}.
\]

The retained-source representation can then be decomposed as 
\begin{equation}
Q^{q,t,s}_{i}=\left(1-\beta^{q,t,s}_{i}\right)Q^{q,t,s}_{i,H}+\beta^{q,t,s}_{i}Q^{q,t,s}_{i,A}.\label{eq:teacher-representation-decomposition}
\end{equation}
Since $Q^{q,t,s}_{i,H}$ and $Q^{q,t,s}_{i,A}$ are convex combinations
of representations satisfying $\eqref{eq:source-representation-bound}$,
\[
\left\Vert Q^{q,t,s}_{i,H}\right\Vert _{F}\leq M_{q},\qquad\left\Vert Q^{q,t,s}_{i,A}\right\Vert _{F}\leq M_{q}.
\]
Therefore, 
\begin{align}
\left\Vert Q^{q,t,s}_{i}-Q^{q,t,s}_{i,H}\right\Vert _{F} & =\beta^{q,t,s}_{i}\left\Vert Q^{q,t,s}_{i,A}-Q^{q,t,s}_{i,H}\right\Vert _{F}\nonumber \\
 & \leq2M_{q}\beta^{q,t,s}_{i}.\label{eq:teacher-representation-perturbation}
\end{align}

If $\beta^{q,t,s}_{i}=0$, no Byzantine source is retained and the
same bound holds trivially. By Assumption\textasciitilde 3, 
\begin{align*}
\left\Vert g^{q,t,s}_{i}-g^{q,t,s}_{i,H}\right\Vert _{2} & =\left\Vert \Gamma^{t,s}_{i,q}\!\left(Q^{q,t,s}_{i}\right)-\Gamma^{t,s}_{i,q}\left(Q^{q,t,s}_{i,H}\right)\right\Vert _{2}\\
 & \leq L_{q}\left\Vert Q^{q,t,s}_{i}-Q^{q,t,s}_{i,H}\right\Vert _{F}\\
 & \leq2L_{q}M_{q}\beta^{q,t,s}_{i},
\end{align*}
 which proves $\eqref{eq:byz-gradient-perturbation}$. Hence, the
effect of the retained Byzantine sources on each distillation gradient
is controlled by their total retained weight.

Finally, by the triangle inequality, 
\begin{align}
 & \sum_{q\in\mathcal{Q}}\omega^{q,t,s}_{i}\left\Vert g^{q,t,s}_{i}\right\Vert _{2}\nonumber \\
\leq & \sum_{q\in\mathcal{Q}}\omega^{q,t,s}_{i}\left\Vert g^{q,t,s}_{i,H}\right\Vert _{2}+\sum_{q\in\mathcal{Q}}\omega^{q,t,s}_{i}\left\Vert g^{q,t,s}_{i}-g^{q,t,s}_{i,H}\right\Vert _{2}\nonumber \\
\leq & G^{H}_{i}+\sum_{q\in\mathcal{Q}}2\omega^{q,t,s}_{i}L_{q}M_{q}\beta^{q,t,s}_{i}\nonumber \\
\leq & G^{H}_{i}+\varepsilon^{\mathrm{Byz}}_{i}=G_{i},
\end{align}
where the second inequality follows from Assumption\textasciitilde 4
and $\eqref{eq:byz-gradient-perturbation}$.

This proves $\eqref{eq:distillation-gradient-bound}$.
\end{IEEEproof}

\subsection{Bound of Private-Gradient Validation}

The next lemma analyzes the safety of private-gradient validation
in the cross-modality fusion process. It shows that the validated
multi-modality gradient remains bounded by $G_{i}$, while its possible
adverse effect on the local objective is bounded by $\delta_{i}G_{i}$.
\begin{lem}
(Safety of private-gradient validation) \label{lem:safe-private-validation}
Under Assumption 2 and Lemma $\ref{lem:byzantine-distillation-bound}$,
for every honest client $i$, training round $t$, and distillation
step $s$, the validated multi-modality gradient satisfies 
\begin{equation}
\left\Vert g^{\mathrm{MG},t,s}_{i}\right\Vert _{2}\leq G_{i}.\label{eq:validated-mg-gradient-bound}
\end{equation}
Moreover, there exists a finite constant $\delta_{i}$ with $\delta^{2}_{i}\leq G^{2}_{g}$
such that 
\begin{equation}
\mathbb{E}\left[\left\langle \nabla F_{i}\left(\widetilde{\theta}^{t,s-1}_{i}\right),g^{\mathrm{MG},t,s}_{i}\right\rangle \mid\mathcal{F}_{t,s}\right]\geq-\delta_{i}G_{i},\label{eq:true-gradient-mg-bound}
\end{equation}
where $\mathcal{F}_{t,s}$ denotes the information available before
sampling the private mini-batch used to compute $g^{\mathrm{sup},t,s}_{i}$.
\end{lem}
\begin{IEEEproof}
We first characterize the estimation error of the private supervised
gradient. At distillation step $s$, $g^{\mathrm{sup},t,s}_{i}$ is
computed from a private mini-batch at the current model $\widetilde{\theta}^{t,s-1}_{i}$.
Conditioning on $\mathcal{F}_{t,s}$ fixes $\widetilde{\theta}^{t,s-1}_{i}$,
and the remaining randomness comes only from the sampled private mini-batch.
By the unbiasedness condition in Assumption 2, 
\begin{equation}
\mathbb{E}\left[g^{\mathrm{sup},t,s}_{i}\mid\mathcal{F}_{t,s}\right]=\nabla F_{i}\left(\widetilde{\theta}^{t,s-1}_{i}\right).\label{eq:private-gradient-unbiased}
\end{equation}
Using this unbiasedness together with the bounded second moment in
Assumption 2 gives 
\begin{align}
 & \mathbb{E}\left[\left\Vert g^{\mathrm{sup},t,s}_{i}-\nabla F_{i}\left(\widetilde{\theta}^{t,s-1}_{i}\right)\right\Vert ^{2}_{2}\mid\mathcal{F}_{t,s}\right]\nonumber \\
= & \mathbb{E}\left[\left\Vert g^{\mathrm{sup},t,s}_{i}\right\Vert ^{2}_{2}\mid\mathcal{F}_{t,s}\right]-\left\Vert \nabla F_{i}\left(\widetilde{\theta}^{t,s-1}_{i}\right)\right\Vert ^{2}_{2}\nonumber \\
\leq & G^{2}_{g}.\label{eq:private-gradient-error-bound}
\end{align}
We denote a uniform upper bound on this mean-square estimation error
by $\delta^{2}_{i}$, so that 
\begin{equation}
\mathbb{E}\left[\left\Vert g^{\mathrm{sup},t,s}_{i}-\nabla F_{i}\left(\widetilde{\theta}^{t,s-1}_{i}\right)\right\Vert ^{2}_{2}\mid\mathcal{F}_{t,s}\right]\leq\delta^{2}_{i},\qquad\delta^{2}_{i}\leq G^{2}_{g}.\label{eq:validation-error-delta}
\end{equation}

We next consider the validated public gradient. By the exact projection
used for the class prediction and boundary decision gradients, their
components opposing the private supervised gradient are removed. The
relation gradient is retained only when it has a positive inner product
with the private supervised gradient. Therefore, the validation mechanism
defined in the method section guarantees 
\begin{equation}
\left\langle g^{\mathrm{sup},t,s}_{i},g^{\mathrm{MG},t,s}_{i}\right\rangle \geq0.\label{eq:mg-validation-alignment}
\end{equation}
The validation operations do not increase the norm of any modality-wise
gradient. For $q\in\{\mathrm{pred},\mathrm{bd}\}$, the exact projection
removes only the component parallel and opposite to $g^{\mathrm{sup},t,s}_{i}$
and hence 
\begin{equation}
\left\Vert \widetilde{g}^{q,t,s}_{i}\right\Vert _{2}\leq\left\Vert g^{q,t,s}_{i}\right\Vert _{2}.\label{eq:projection-norm-nonincreasing}
\end{equation}
For the prediction correlation modality, $\widetilde{g}^{\mathrm{rel},t,s}_{i}$
is either $g^{\mathrm{rel},t,s}_{i}$ or the zero vector, and therefore
\begin{equation}
\left\Vert \widetilde{g}^{\mathrm{rel},t,s}_{i}\right\Vert _{2}\leq\left\Vert g^{\mathrm{rel},t,s}_{i}\right\Vert _{2}.\label{eq:relation-gating-norm}
\end{equation}
Since all modality weights are non-negative, the triangle inequality
and Lemma$~\ref{lem:byzantine-distillation-bound}$ yield 
\begin{align}
\left\Vert g^{\mathrm{MG},t,s}_{i}\right\Vert _{2} & =\left\Vert \sum_{q\in\mathcal{Q}}\omega^{q,t,s}_{i}\widetilde{g}^{q,t,s}_{i}\right\Vert _{2}\nonumber \\
 & \leq\sum_{q\in\mathcal{Q}}\omega^{q,t,s}_{i}\left\Vert \widetilde{g}^{q,t,s}_{i}\right\Vert _{2}\nonumber \\
 & \leq\sum_{q\in\mathcal{Q}}\omega^{q,t,s}_{i}\left\Vert g^{q,t,s}_{i}\right\Vert _{2}\nonumber \\
 & \leq G_{i}.\label{eq:mg-gradient-bound-proof}
\end{align}
This proves Eq.$~\eqref{eq:validated-mg-gradient-bound}$.

Finally, we relate the validated public gradient to the true local
gradient. We write 
\begin{align}
 & \left\langle \nabla F_{i}\left(\widetilde{\theta}^{t,s-1}_{i}\right),g^{\mathrm{MG},t,s}_{i}\right\rangle \nonumber \\
= & \left\langle g^{\mathrm{sup},t,s}_{i},g^{\mathrm{MG},t,s}_{i}\right\rangle +\left\langle \nabla F_{i}\left(\widetilde{\theta}^{t,s-1}_{i}\right)-g^{\mathrm{sup},t,s}_{i},g^{\mathrm{MG},t,s}_{i}\right\rangle .\label{eq:true-gradient-decomposition}
\end{align}
Using Eq.$\eqref{eq:mg-validation-alignment}$ and the Cauchy-{}-Schwarz
inequality, 
\begin{align}
 & \left\langle \nabla F_{i}\left(\widetilde{\theta}^{t,s-1}_{i}\right),g^{\mathrm{MG},t,s}_{i}\right\rangle \nonumber \\
\geq & -\left\Vert \nabla F_{i}\left(\widetilde{\theta}^{t,s-1}_{i}\right)-g^{\mathrm{sup},t,s}_{i}\right\Vert _{2}\left\Vert g^{\mathrm{MG},t,s}_{i}\right\Vert _{2}\nonumber \\
\geq & -G_{i}\left\Vert \nabla F_{i}\left(\widetilde{\theta}^{t,s-1}_{i}\right)-g^{\mathrm{sup},t,s}_{i}\right\Vert _{2}.\label{eq:true-gradient-pointwise-bound}
\end{align}
Taking conditional expectation and applying Jensen's inequality together
with Eq.$\eqref{eq:validation-error-delta}$, we obtain 
\begin{align}
 & \mathbb{E}\left[\left\langle \nabla F_{i}\left(\widetilde{\theta}^{t,s-1}_{i}\right),g^{\mathrm{MG},t,s}_{i}\right\rangle \mid\mathcal{F}_{t,s}\right]\nonumber \\
\geq & -G_{i}\mathbb{E}\left[\left\Vert \nabla F_{i}\left(\widetilde{\theta}^{t,s-1}_{i}\right)-g^{\mathrm{sup},t,s}_{i}\right\Vert _{2}\mid\mathcal{F}_{t,s}\right]\nonumber \\
\geq & -G_{i}\sqrt{\mathbb{E}\left[\left\Vert \nabla F_{i}\left(\widetilde{\theta}^{t,s-1}_{i}\right)-g^{\mathrm{sup},t,s}_{i}\right\Vert ^{2}_{2}\mid\mathcal{F}_{t,s}\right]}\nonumber \\
\geq & -\delta_{i}G_{i}.
\end{align}
Thus, although the private supervised gradient is only a stochastic
estimate of the true local gradient, the possible adverse effect of
the validated public-distillation direction is bounded by $\delta_{i}G_{i}$.

This completes the proof.
\end{IEEEproof}

\subsection{Proof of the Main Result}

We now prove Theorem $\ref{thm:main-convergence}$ by combining the
descent achieved during local supervised training with the bounded
effect of the subsequent public-distillation stage. The local-training
analysis gives the local-update drift and stochastic-gradient terms,
while Lemmas$~\ref{lem:byzantine-distillation-bound}$ and$~\ref{lem:safe-private-validation}$
control the two public-distillation terms.
\begin{IEEEproof}
The proof consists of three steps. (A) We use the local-SGD properties
established in our previous work \cite{ma2026rdpfl} to bound the
change of the local supervised objective after the $K$ private updates.
(B) Lemma $\ref{lem:byzantine-distillation-bound}$ provides the bound
$G_{i}$ on the Byzantine-affected distillation gradients, and Lemma
$\ref{lem:safe-private-validation}$ uses this bound to control the
subsequent public-distillation updates. (C) We combine the two stage-wise
bounds over one training round and telescope the resulting inequality
over $T$ rounds.

(A) For the local supervised stage, define the average local stochastic
gradient as 
\begin{equation}
u^{t}_{i}=\frac{1}{K}\sum^{K-1}_{k=0}g_{i}\left(\theta^{t,k}_{i};\xi^{t,k}_{i}\right),
\end{equation}
so that 
\begin{equation}
\bar{\theta}^{t}_{i}=\theta^{t}_{i}-\gamma u^{t}_{i},\qquad\gamma=\eta K.
\end{equation}
Under Assumptions\textasciitilde 1-{}-2, the local-SGD result established
in \cite{ma2026rdpfl} gives 
\begin{equation}
\mathbb{E}\left[u^{t}_{i}\mid\theta^{t}_{i}\right]=\nabla F_{i}(\theta^{t}_{i})+b^{t}_{i},
\end{equation}
with 
\begin{equation}
\left\Vert b^{t}_{i}\right\Vert _{2}\leq\frac{LG_{g}}{2}\eta K,\quad\mathbb{E}\left[\left\Vert u^{t}_{i}-\mathbb{E}[u^{t}_{i}\mid\theta^{t}_{i}]\right\Vert ^{2}_{2}\mid\theta^{t}_{i}\right]\leq G^{2}_{g}.\label{eq:local-sgd-properties}
\end{equation}
 The bias bound follows from $L$-smoothness and $\mathbb{E}\|\theta^{t,k}_{i}-\theta^{t}_{i}\|_{2}\leq\eta kG_{g}$,
averaged over $k=0,\ldots,K-1$; Jensen's inequality and Assumption
2 give the second-moment bound. By the $L$-smoothness of $F_{i}$,
\begin{equation}
F_{i}(\bar{\theta}^{t}_{i})\leq F_{i}(\theta^{t}_{i})-\gamma\left\langle \nabla F_{i}(\theta^{t}_{i}),u^{t}_{i}\right\rangle +\frac{L\gamma^{2}}{2}\left\Vert u^{t}_{i}\right\Vert ^{2}_{2}.
\end{equation}
Taking expectation and using $\eqref{eq:local-sgd-properties}$, together
with $|\langle a,b\rangle|\leq\frac{1}{4}\|a\|^{2}_{2}+\|b\|^{2}_{2}$
and $L\gamma\leq1/4$, gives 
\begin{align}
\mathbb{E}\left[F_{i}(\bar{\theta}^{t}_{i})\right] & \leq\mathbb{E}\left[F_{i}(\theta^{t}_{i})\right]-\frac{\gamma}{2}\mathbb{E}\left[\left\Vert \nabla F_{i}(\theta^{t}_{i})\right\Vert ^{2}_{2}\right]\label{eq:local-stage-final-bound}\\
 & \quad+\frac{5\gamma}{16}L^{2}G^{2}_{g}\eta^{2}K^{2}+\frac{L\gamma^{2}}{2}G^{2}_{g}.\nonumber 
\end{align}

(B) We now analyze the public-distillation stage. Let $\widetilde{\theta}^{t,0}_{i}=\bar{\theta}^{t}_{i}$,
and for $s=1,\ldots,S$ let 
\begin{equation}
\widetilde{\theta}^{t,s}_{i}=\widetilde{\theta}^{t,s-1}_{i}-\alpha_{t}g^{\mathrm{MG},t,s}_{i},\qquad\alpha_{t}=\eta_{d}\lambda_{t}.\label{eq:public-update-main-proof}
\end{equation}
The model at the beginning of the next training round is $\theta^{t+1}_{i}=\widetilde{\theta}^{t,S}_{i}$.
By the $L$-smoothness of $F_{i}$, 
\begin{align}
F_{i}\left(\widetilde{\theta}^{t,s}_{i}\right) & \leq F_{i}\left(\widetilde{\theta}^{t,s-1}_{i}\right)-\alpha_{t}\left\langle \nabla F_{i}\left(\widetilde{\theta}^{t,s-1}_{i}\right),g^{\mathrm{MG},t,s}_{i}\right\rangle \label{eq:public-smoothness-step}\\
 & \quad+\frac{L\alpha^{2}_{t}}{2}\left\Vert g^{\mathrm{MG},t,s}_{i}\right\Vert ^{2}_{2}.\nonumber 
\end{align}

By Lemma $\ref{lem:byzantine-distillation-bound}$, the multi-modality
distillation gradients are bounded by $G_{i}$. Lemma $\ref{lem:safe-private-validation}$
further gives 
\begin{align*}
 & \mathbb{E}\left[\left\langle \nabla F_{i}\left(\widetilde{\theta}^{t,s-1}_{i}\right),g^{\mathrm{MG},t,s}_{i}\right\rangle \mid\mathcal{F}_{t,s}\right]\\
\geq & -\delta_{i}G_{i},\qquad\left\Vert g^{\mathrm{MG},t,s}_{i}\right\Vert _{2}\\
\leq & G_{i}.
\end{align*}

Taking conditional expectation in Eq.$~\eqref{eq:public-smoothness-step}$
therefore gives 
\begin{equation}
\mathbb{E}\left[F_{i}\left(\widetilde{\theta}^{t,s}_{i}\right)\mid\mathcal{F}_{t,s}\right]\leq F_{i}\left(\widetilde{\theta}^{t,s-1}_{i}\right)+\alpha_{t}\delta_{i}G_{i}+\frac{L\alpha^{2}_{t}}{2}G^{2}_{i}.\label{eq:public-single-step-bound}
\end{equation}
Applying this bound successively for $s=1,\ldots,S$ and using the
tower property, together with $\widetilde{\theta}^{t,0}_{i}=\bar{\theta}^{t}_{i}$
and $\widetilde{\theta}^{t,S}_{i}=\theta^{t+1}_{i}$, yields 
\begin{equation}
\mathbb{E}\left[F_{i}\left(\theta^{t+1}_{i}\right)\right]\leq\mathbb{E}\left[F_{i}\left(\bar{\theta}^{t}_{i}\right)\right]+S\alpha_{t}\delta_{i}G_{i}+\frac{LS\alpha^{2}_{t}}{2}G^{2}_{i}.\label{eq:public-stage-final-bound}
\end{equation}

(C) Combining $\eqref{eq:local-stage-final-bound}$ and $\eqref{eq:public-stage-final-bound}$,
and using $\alpha_{t}\leq\alpha$, gives 
\begin{align}
\mathbb{E}\left[F_{i}\left(\theta^{t+1}_{i}\right)\right]\leq & \mathbb{E}\left[F_{i}\left(\theta^{t}_{i}\right)\right]-\frac{\gamma}{2}\mathbb{E}\left[\left\Vert \nabla F_{i}\left(\theta^{t}_{i}\right)\right\Vert ^{2}_{2}\right]\nonumber \\
 & \quad+\frac{5\gamma}{16}L^{2}G^{2}_{g}\eta^{2}K^{2}+\frac{L\gamma^{2}}{2}G^{2}_{g}\nonumber \\
 & \quad+S\alpha\delta_{i}G_{i}+\frac{LS\alpha^{2}}{2}G^{2}_{i}.\label{eq:one-round-main-bound}
\end{align}
Summing $\eqref{eq:one-round-main-bound}$ over $t=0,\ldots,T-1$
yields 
\begin{align*}
 & \frac{\gamma}{2}\sum^{T-1}_{t=0}\mathbb{E}\left[\left\Vert \nabla F_{i}\left(\theta^{t}_{i}\right)\right\Vert ^{2}_{2}\right]\\
\leq & F_{i}\left(\theta^{0}_{i}\right)-\mathbb{E}\left[F_{i}\left(\theta^{T}_{i}\right)\right]\\
 & \quad+\frac{5\gamma T}{16}L^{2}G^{2}_{g}\eta^{2}K^{2}+\frac{L\gamma^{2}T}{2}G^{2}_{g}\\
 & \quad+TS\alpha\delta_{i}G_{i}+\frac{LTS\alpha^{2}}{2}G^{2}_{i}.
\end{align*}
Since $F_{i}(\theta)\geq F^{\inf}_{i}$ by Assumption\textasciitilde 1,
\begin{align*}
 & \frac{\gamma}{2}\sum^{T-1}_{t=0}\mathbb{E}\left[\left\Vert \nabla F_{i}\left(\theta^{t}_{i}\right)\right\Vert ^{2}_{2}\right]\\
\leq & F_{i}\left(\theta^{0}_{i}\right)-F^{\inf}_{i}\\
 & \quad+\frac{5\gamma T}{16}L^{2}G^{2}_{g}\eta^{2}K^{2}+\frac{L\gamma^{2}T}{2}G^{2}_{g}\\
 & \quad+TS\alpha\delta_{i}G_{i}+\frac{LTS\alpha^{2}}{2}G^{2}_{i}.
\end{align*}
Dividing both sides by $\gamma T/2$ gives 
\begin{align*}
 & \frac{1}{T}\sum^{T-1}_{t=0}\mathbb{E}\left[\left\Vert \nabla F_{i}\left(\theta^{t}_{i}\right)\right\Vert ^{2}_{2}\right]\\
\leq & \frac{2\left(F_{i}\left(\theta^{0}_{i}\right)-F^{\inf}_{i}\right)}{\gamma T}+\frac{5}{8}L^{2}G^{2}_{g}\eta^{2}K^{2}\\
 & \quad+L\gamma G^{2}_{g}+\frac{2S\alpha\delta_{i}G_{i}}{\gamma}+\frac{LS\alpha^{2}G^{2}_{i}}{\gamma}.
\end{align*}
This is the bound stated in Theorem$~\ref{thm:main-convergence}$.
Moreover, since $G_{i}=G^{H}_{i}+\varepsilon^{\mathrm{Byz}}_{i}$,
smaller residual Byzantine weights reduce $\varepsilon^{\mathrm{Byz}}_{i}$
and tighten the Byzantine-dependent terms in the bound.
\end{IEEEproof}

\section{Experiments\label{sec:Experiments}}

\subsection{Experimental Setup}

\subsubsection{Datasets}

We evaluate the proposed method on CIFAR-10 and CIFAR-100. For each
dataset, $10\%$ of the official training set is used as an unlabeled
public dataset, $5\%$ is used for validation, and the remaining samples
are partitioned among clients as private labeled data. The official
test set is used only for final evaluation. Labels of the public samples
are not used during teacher construction or distillation.

\subsubsection{Federated Setting}

We simulate a fully connected decentralized system with $N=10$ clients
on a single machine. No central server performs model, gradient, or
prediction aggregation. The coordinator only schedules computation
and routes messages between clients. Among the ten clients, three
are Byzantine: 
\[
\mathcal{A}=\{7,8,9\},
\]
and the remaining clients are honest: 
\[
\mathcal{H}=\{0,1,2,3,4,5,6\}.
\]
Only honest clients are included in the reported evaluation metrics.

\subsubsection{Data and Model Heterogeneity}

We use label-skewed Dirichlet partitioning to generate non-IID private
data. Unless otherwise specified, the concentration parameter is set
to $\alpha=0.5$. Each client receives at least 100 private training
samples. To introduce model heterogeneity, clients use CIFAR-compatible
variants of ResNet-18\cite{he2016deep}, MobileNetV2\cite{sandler2018mobilenetv2},
ShuffleNetV2\cite{ma2018shufflenet}, and VGG-11\cite{simonyan2015very}.
Different architectures are assigned to different clients as shown
in Table \ref{tab:Model-Architecture}.

\begin{table}[tbh]
\caption{Heterogeneous model architectures across clients.\label{tab:Model-Architecture}}

\centering{}%
\begin{tabular}{|c|c|}
\hline 
Client & Architecture\tabularnewline
\hline 
\hline 
Client 0 \& Client 4 \& Client 8 & ResNet-18\tabularnewline
\hline 
Client 1 \& Client 5 \& Client 9 & MobileNetV2\tabularnewline
\hline 
Client 2 \& Client 6 & ShuffleNetV2\tabularnewline
\hline 
Client 3 \& Client 7 & VGG-11\tabularnewline
\hline 
\end{tabular}
\end{table}

\subsubsection{Training and Distillation Settings}

Each experiment runs for 120 training rounds. In each round, every
client first performs one epoch of supervised training on its private
data. The local optimizer is SGD with momentum $0.9$, batch size
$64$, initial learning rate $0.05$, and weight decay $5\times10^{-4}$.
The learning rate is multiplied by $0.1$ at $50\%$ and $75\%$ of
the total training rounds.

The proposed method then processes two public mini-batches per round,
each with batch size $256$. Public logits are standardized row-wise
before reliability estimation. Class prediction, boundary decision,
and prediction correlation discrepancies are used for client selection
and weighting, followed by the construction of three distillation
teachers. The three distillation gradients are validated using a private
supervised gradient before they are combined for the public update.
The principal distillaiton hyperparameters are listed in Table \ref{tab:hyperparameters_distillation}.

\begin{table}[tbh]
\caption{Hyperparameters of Distillation Setting.\label{tab:hyperparameters_distillation}}

\centering{}%
\begin{tabular}{|c|c|}
\hline 
Hyperparameters & Value\tabularnewline
\hline 
\hline 
$T_{\mathrm{kd}}$ & 2.0\tabularnewline
\hline 
$\eta_{\mathrm{kd}}$ & 0.01\tabularnewline
\hline 
$\tau$ & 0.5\tabularnewline
\hline 
$\lambda$ & 1.0\tabularnewline
\hline 
$\epsilon$ & $10^{-8}$\tabularnewline
\hline 
\end{tabular}
\end{table}

\subsubsection{Byzantine Attacks}

We evaluate four Byzantine attacks together with a benign setting.
In the Gaussian attack, malicious clients replace their outgoing messages
with random Gaussian values. The Sign-flip attack reverses the sign
of the transmitted values. The Targeted attack increases the preference
for a selected target class, while the Bias attack adds a fixed class-wise
bias to the outgoing messages.

For the proposed method and other output-space methods, the attacks
are applied to the exchanged logits. Byzantine clients may send different
manipulated logits to different honest clients.

\subsubsection{Evaluation Metrics}

Test performance is reported only for honest clients. For honest client
$i$, let $Acc_{i}$ denote its test accuracy. We report 
\begin{equation}
Acc_{\mathrm{mean}}=\frac{1}{|\mathcal{H}|}\sum_{i\in\mathcal{H}}Acc_{i},\qquad Acc_{\mathrm{worst}}=\min_{i\in\mathcal{H}}Acc_{i}.
\end{equation}
$Acc_{\mathrm{mean}}$ measures the overall performance of honest
clients, while $Acc_{\mathrm{worst}}$ measures the performance of
the most affected honest client. The latter is particularly useful
when Byzantine clients send different messages to different receivers.

\subsubsection{Compared Methods}

The primary comparison includes public-distillation methods that support
the same heterogeneous model assignment, private-data partition, public
batches, training schedule, and logit-space Byzantine attacks:
\begin{itemize}
\item FedMD\cite{li2019fedmd}: prediction-based distillation using the
arithmetic mean of the available source logits;
\item FedDF\cite{lin2020ensemble}: public ensemble distillation using the
arithmetic mean of source probabilities without parameter averaging;
\item Ours: the proposed robust decentralized federated distillation method
with multi-modality knowledge collaboration.
\end{itemize}

\subsection{Main Comparison under Heterogeneous Models}

\subsubsection{CIFAR-10 Results}

Table $\ref{tab:exp-main-cifar10}$ reports the CIFAR-10 test results.
Ours achieves the highest mean accuracy under None, Gaussian, and
Sign-flip, while FedDF and FedMD obtain the highest mean accuracy
under Targeted and Bias, respectively. More importantly, Ours achieves
the highest worst-client accuracy in all five distillation conditions.

\begin{table}[tbh]
\caption{CIFAR-10 accuracy ($\%$) with heterogeneous models and Dirichlet
non-IID data. ''Mean'' and ''Worst'' denote mean and worst honest-client
accuracy.\label{tab:exp-main-cifar10}}

\centering{}%
\begin{tabular}{|c|c|c|c|c|c|c|}
\hline 
\multirow{2}{*}{Attack} & \multicolumn{2}{c|}{FedMD} & \multicolumn{2}{c|}{FedDF} & \multicolumn{2}{c|}{Ours}\tabularnewline
\cline{2-7}
 & Mean & Worst & Mean & Worst & Mean & Worst\tabularnewline
\hline 
None & 51.28 & 32.43 & 52.08 & 39.71 & 54.79 & 45.66\tabularnewline
\hline 
Gaussian & 48.96 & 33.94 & 50.31 & 34.92 & 51.92 & 42.13\tabularnewline
\hline 
Sign-flip & 49.09 & 35.27 & 49.72 & 35.94 & 50.01 & 39.33\tabularnewline
\hline 
Targeted & 51.57 & 33.69 & 52.66 & 38.33 & 51.50 & 39.90\tabularnewline
\hline 
Bias & 53.44 & 39.51 & 51.94 & 36.89 & 52.57 & 42.23\tabularnewline
\hline 
Average & 50.87 & 34.97 & 51.34 & 37.16 & 52.16 & 41.85\tabularnewline
\hline 
\end{tabular}
\end{table}

Averaged over the five conditions, Ours improves mean accuracy by
$0.82$ percentage points over FedDF, the stronger primary baseline
on this aggregate metric. Its average worst-client accuracy is $41.85\%$,
which is $4.69$ points above FedDF and $6.88$ points above FedMD.
The larger gain in worst-client accuracy indicates that receiver-specific
robust teacher construction primarily benefits vulnerable clients
rather than only the best-performing model.

\subsubsection{CIFAR-100 Results}

Table $\ref{tab:exp-main-cifar100}$ presents the CIFAR-100 results.
Ours obtains the highest mean and worst-client accuracy under every
evaluated attack. Averaged across all five conditions, it improves
mean accuracy by $0.96$ points over FedMD and improves worst-client
accuracy by $1.79$ points over FedMD, which is the stronger baseline
according to the corresponding aggregate metrics.

\begin{table}[tbh]
\caption{CIFAR-100 accuracy ($\%$) with heterogeneous models and Dirichlet
non-IID data.\label{tab:exp-main-cifar100}}

\centering{}%
\begin{tabular}{|c|c|c|c|c|c|c|}
\hline 
\multirow{2}{*}{Attack} & \multicolumn{2}{c|}{FedMD} & \multicolumn{2}{c|}{FedDF} & \multicolumn{2}{c|}{Ours}\tabularnewline
\cline{2-7}
 & Mean & Worst & Mean & Worst & Mean & Worst\tabularnewline
\hline 
None & 25.13 & 20.56 & 25.24 & 20.84 & 26.01 & 22.87\tabularnewline
\hline 
Gaussian & 24.47 & 21.07 & 24.09 & 18.20 & 25.92 & 22.98\tabularnewline
\hline 
Sign-flip & 23.73 & 20.21 & 24.08 & 19.48 & 25.46 & 22.79\tabularnewline
\hline 
Targeted & 25.46 & 21.44 & 24.87 & 18.35 & 25.87 & 22.74\tabularnewline
\hline 
Bias & 25.45 & 21.02 & 25.24 & 20.62 & 25.80 & 21.85\tabularnewline
\hline 
Average & 24.85 & 20.86 & 24.70 & 19.50 & 25.81 & 22.65\tabularnewline
\hline 
\end{tabular}
\end{table}

\begin{figure}[tbh]
\begin{centering}
\includegraphics[scale=0.45]{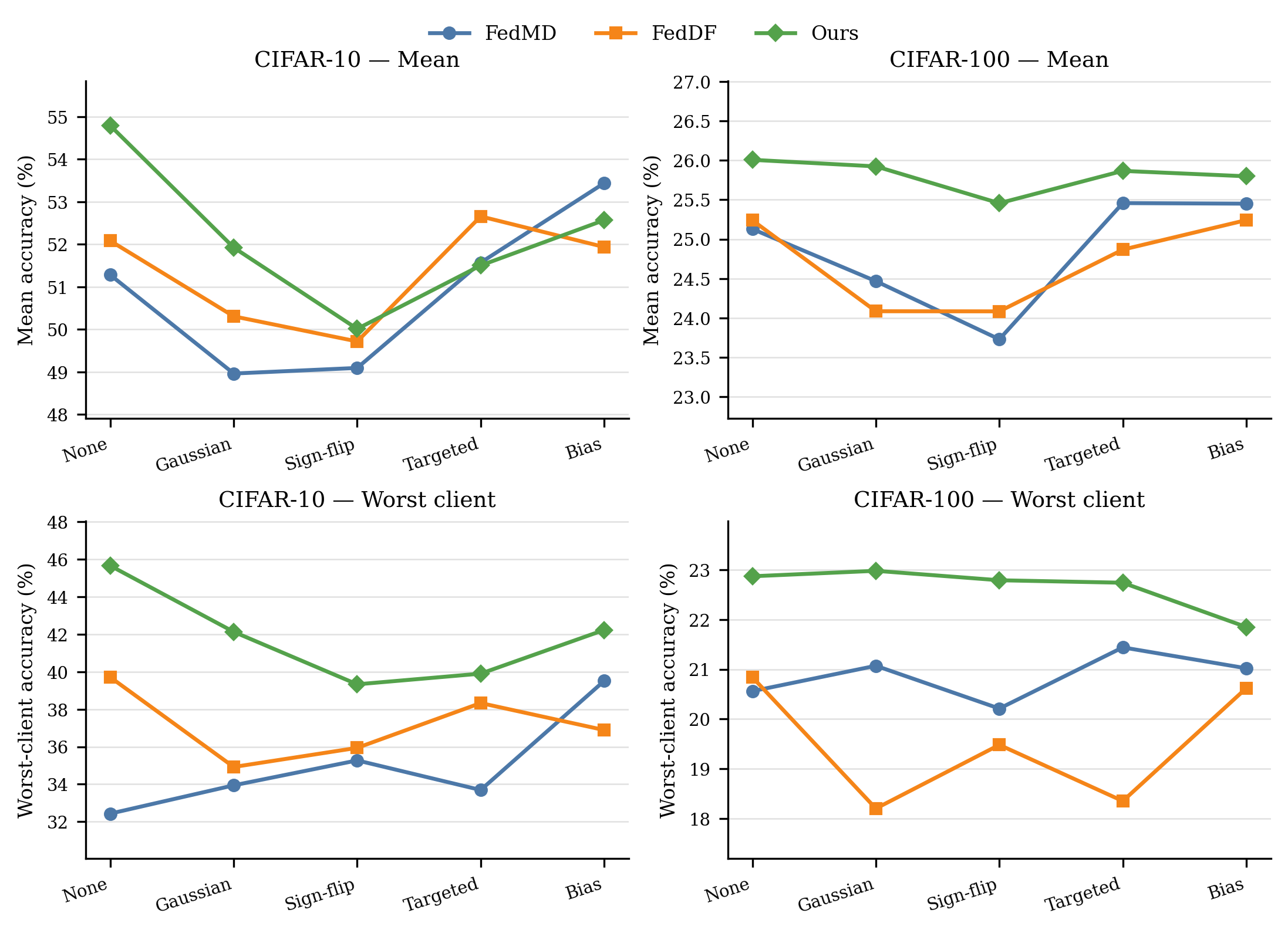}
\par\end{centering}
\caption{Comparison of FedMD, FedDF, and Ours under the benign setting and
four Byzantine attacks. The top and bottom rows show mean and worst-client
accuracy, respectively. Each panel uses an independently scaled vertical
axis to expose relative differences; exact values are reported in
Tables$~\ref{tab:exp-main-cifar10}$ and $\ref{tab:exp-main-cifar100}$.\label{fig:exp-main-comparison}}
\end{figure}

Figure $\ref{fig:exp-main-comparison}$ visualizes the main comparison.
The separation between Ours and the two baselines is more consistent
for worst-client accuracy than for mean accuracy, especially on CIFAR-10.

\begin{table*}[tbh]
\caption{Per-client Ours accuracy ($\%$) on CIFAR-10. C0 and C4 use ResNet-18;
C1 and C5 use MobileNetV2; C2 and C6 use ShuffleNetV2; C3 uses VGG-11.\label{tab:exp-client-cifar10}}

\centering{}%
\begin{tabular}{|c|c|c|c|c|c|c|c|c|c|}
\hline 
Attack & C0 & C1 & C2 & C3 & C4 & C5 & C6 & Mean & Worst\tabularnewline
\hline 
\hline 
None & 83.45 & 46.94 & 49.51 & 50.62 & 54.46 & 45.66 & 52.89 & 54.79 & 45.66\tabularnewline
\hline 
Gaussian & 81.15 & 43.85 & 47.86 & 47.15 & 49.32 & 42.13 & 52.00 & 51.92 & 42.13\tabularnewline
\hline 
Sign-flip & 81.36 & 39.33 & 42.15 & 44.87 & 49.19 & 40.02 & 53.16 & 50.01 & 39.33\tabularnewline
\hline 
Targeted & 82.34 & 43.72 & 49.13 & 48.02 & 48.32 & 39.90 & 49.08 & 51.50 & 39.90\tabularnewline
\hline 
Bias & 80.53 & 42.23 & 49.65 & 49.75 & 50.69 & 42.68 & 52.44 & 52.57 & 42.23\tabularnewline
\hline 
\end{tabular}
\end{table*}

\begin{table*}[tbh]
\caption{Per-client Ours accuracy ($\%$) on CIFAR-100.\label{tab:exp-client-cifar100}}

\centering{}%
\begin{tabular}{|c|c|c|c|c|c|c|c|c|c|}
\hline 
Attack & C0 & C1 & C2 & C3 & C4 & C5 & C6 & Mean & Worst\tabularnewline
\hline 
\hline 
None & 26.31 & 22.87 & 26.32 & 27.08 & 27.19 & 27.06 & 25.21 & 26.01 & 22.87\tabularnewline
\hline 
Gaussian & 26.73 & 22.98 & 25.97 & 26.94 & 27.44 & 26.07 & 25.33 & 25.92 & 22.98\tabularnewline
\hline 
Sign-flip & 26.23 & 22.79 & 25.66 & 26.15 & 26.51 & 25.98 & 24.87 & 25.46 & 22.79\tabularnewline
\hline 
Targeted & 26.42 & 22.74 & 25.58 & 27.21 & 27.28 & 27.26 & 24.57 & 25.87 & 22.74\tabularnewline
\hline 
Bias & 25.90 & 21.85 & 26.24 & 27.17 & 27.49 & 27.03 & 24.90 & 25.80 & 21.85\tabularnewline
\hline 
\end{tabular}
\end{table*}

\subsection{Robustness across Attack Types}

On CIFAR-10, the no-attack mean and worst-client accuracies of Ours
are $54.79\%$ and $45.66\%$. Gaussian, Sign-flip, Targeted, and
Bias reduce mean accuracy by $2.87$, $4.78$, $3.29$, and $2.22$
points, respectively. The corresponding worst-client reductions are
$3.53$, $6.33$, $5.76$, and $3.43$ points. Sign-flip is therefore
the strongest tested attack against Ours on CIFAR-10.

The CIFAR-100 results are considerably more stable across attacks.
Relative to the no-attack result, the mean-accuracy changes under
Gaussian, Sign-flip, Targeted, and Bias are $-0.08$, $-0.55$, $-0.14$,
and $-0.21$ points. The worst-client changes are $+0.11$, $-0.08$,
$-0.13$, and $-1.02$ points, respectively. The small positive Gaussian
difference is within normal single-run variation and should not be
interpreted as an improvement caused by the attack.

The validation-selected rounds provide an additional indication of
training stability. On CIFAR-10, the selected rounds are 120, 90,
85, 90, and 120 for None, Gaussian, Sign-flip, Targeted, and Bias,
respectively. In contrast, all CIFAR-100 runs select round 120. The
earlier CIFAR-10 checkpoints under the three more disruptive attacks
suggest that late-stage public supervision can accumulate attack-dependent
noise, even though validation-based model selection limits its effect
on the reported test result.

\subsection{Client-Level Performance}

Tables $\ref{tab:exp-client-cifar10}$ and $\ref{tab:exp-client-cifar100}$
report each honest client's accuracy. These results expose behavior
that is hidden by a single aggregate mean.

\begin{figure}[tbh]
\begin{centering}
\includegraphics[scale=0.45]{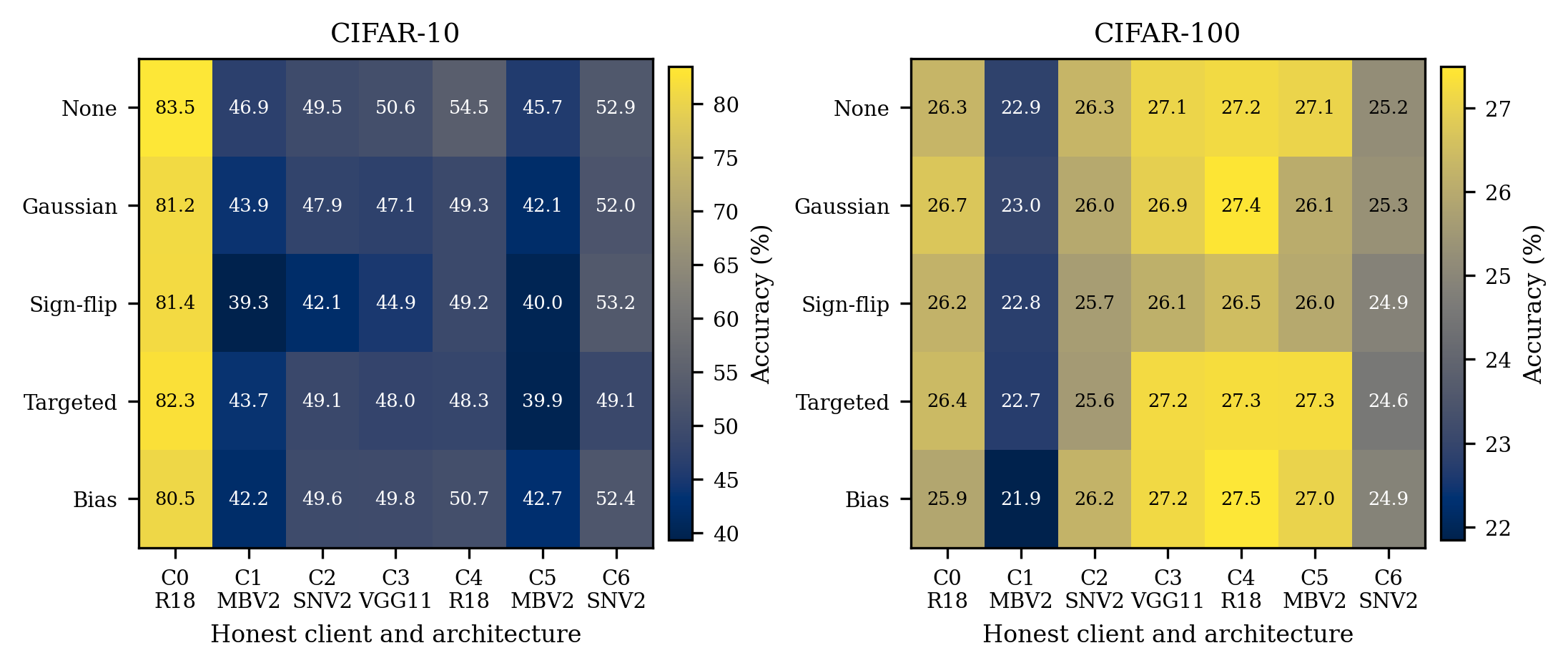}
\par\end{centering}
\caption{Per-client Ours accuracy under different attacks. Model abbreviations
are R18 (ResNet-18), MBV2 (MobileNetV2), SNV2 (ShuffleNetV2), and
VGG11 (VGG-11). Each dataset uses its own color scale.\label{fig:exp-client-heatmaps}}
\end{figure}

\begin{table}[tbh]
\caption{Variant comparison in mean/worst honest-client accuracy ($\%$). ``Attack
avg.'' averages the four Byzantine attacks.\label{tab:exp-variant-analysis}}

\centering{}%
\begin{tabular}{|c|c|c|c|c|c|}
\hline 
\multirow{2}{*}{Dataset} & \multirow{2}{*}{Method} & \multicolumn{2}{c|}{None} & \multicolumn{2}{c|}{Attack avg.}\tabularnewline
\cline{3-6}
 &  & Mean & Worst & Mean & Worst\tabularnewline
\hline 
\multirow{4}{*}{CIFAR-10} & Ours-Base & 50.08 & 37.48 & 49.97 & 36.73\tabularnewline
\cline{2-6}
 & Ours-MM & 49.95 & 34.69 & 50.95 & 37.22\tabularnewline
\cline{2-6}
 & Ours-MM+CA & 50.61 & 35.66 & 50.88 & 35.56\tabularnewline
\cline{2-6}
 & Ours & \textbf{54.79} & \textbf{45.66} & \textbf{51.50} & \textbf{40.90}\tabularnewline
\hline 
\hline 
\multirow{4}{*}{CIFAR-100} & Ours-Base & 24.46 & 21.28 & 24.39 & \multicolumn{1}{c}{20.82}\tabularnewline
\cline{2-6}
 & Ours-MM & 24.79 & 21.06 & 24.59 & 20.52\tabularnewline
\cline{2-6}
 & Ours-MM+CA & 24.96 & 21.68 & 24.80 & 20.39\tabularnewline
\cline{2-6}
 & Ours & \textbf{26.01} & \textbf{22.87} & \textbf{25.76} & \textbf{22.59}\tabularnewline
\hline 
\end{tabular}
\end{table}

\begin{table*}[tbh]
\caption{Final-round Ours diagnostics without attack.\label{tab:exp-mechanism-diagnostics}}

\centering{}%
\begin{tabular}{|c|c|c|c|c|c|c|c|}
\hline 
Dataset & Pseudo cov. & Conf. & $\mathrm{pred}$ & $\mathrm{bd}$ & $\mathrm{rel}$ & $Proj.pred.$ & $Proj.bd.$\tabularnewline
\hline 
\hline 
CIFAR-10 & $67.49\%$ & 0.739 & 0.745 & 0.209 & 0.046 & $35.71\%$ & $42.86\%$\tabularnewline
\hline 
CIFAR-100 & $22.72\%$ & 0.693 & 0.734 & 0.219 & 0.047 & $78.57\%$ & $50.00\%$\tabularnewline
\hline 
\end{tabular}
\end{table*}

CIFAR-10 exhibits substantial client heterogeneity. Across the five
conditions, C0 averages $81.77\%$, whereas C1 and C5 average $43.21\%$
and $42.08\%$. Because C0 and C4 both use ResNet-18 but obtain markedly
different accuracies, architecture alone cannot explain the gap; the
Dirichlet private partitions and client-specific optimization trajectories
are also major factors. On CIFAR-100, client averages across attacks
lie between $22.65\%$ and $27.18\%$, producing a substantially narrower
client-performance range.

The same pattern is visible in Figure $\ref{fig:exp-client-heatmaps}$:
CIFAR-10 performance is dominated by both client identity and private-data
heterogeneity, whereas CIFAR-100 is more uniform across honest clients.

\subsection{Variant Analysis\label{subsec:variant-analysis}}

Table $\ref{tab:exp-variant-analysis}$ compares four versions under
identical heterogeneous settings. ``Attack average'' denotes the macro-average
over Gaussian, Sign-flip, Targeted, and Bias, excluding None. This
is a variant comparison rather than a one-factor-at-a-time ablation
because successive versions modify more than one mechanism.

Ours-Base denotes the original single-modality robust distillation
method. Ours-MM introduces multi-modality knowledge transfer. Ours-MM+CA
further incorporates class-aware boundary modeling and modality-wise
gradient validation. Ours denotes the complete method with confidence-aware
dual targets and conflict-preserving gradient projection.

\begin{figure}[tbh]
\begin{centering}
\includegraphics[scale=0.45]{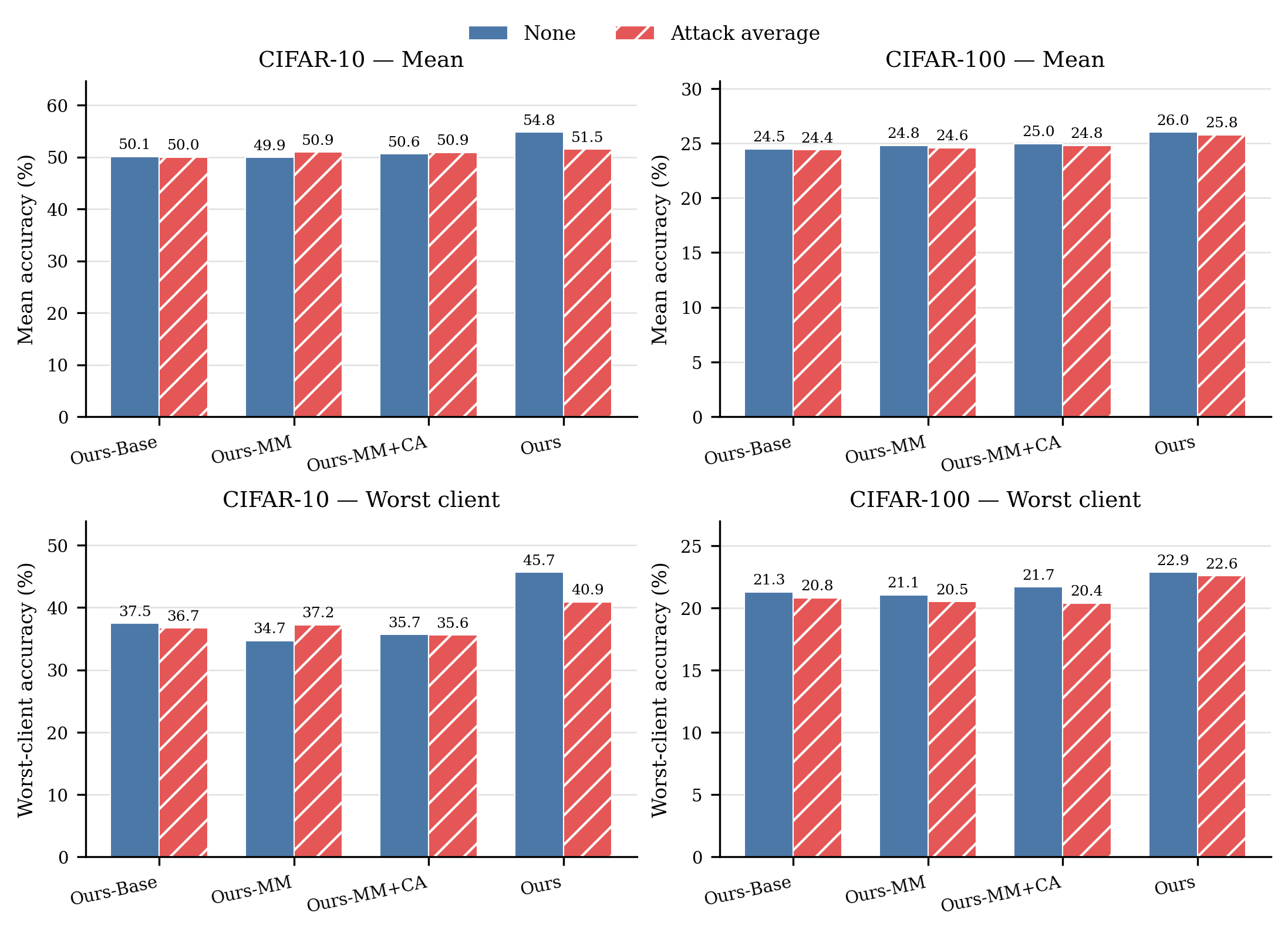}
\par\end{centering}
\caption{Comparison of Ours-Base and the three multi-modality variants. ``Attack
average'' is the macro-average over Gaussian, Sign-flip, Targeted,
and Bias. All bar charts start at zero.\label{fig:exp-variant-comparison}}
\end{figure}

On CIFAR-10 without attack, Ours improves the mean and worst-client
accuracies over Ours-MM+CA by $4.18$ and $10.00$ percentage points,
respectively. Averaged over the four Byzantine attacks, the corresponding
gains are $0.62$ and $5.34$ percentage points. On CIFAR-100, Ours
surpasses Ours-MM+CA by $0.96$ and $2.20$ percentage points in attack-averaged
mean and worst-client accuracy, respectively. These results show that
the complete design primarily improves robustness for the most vulnerable
honest client, consistent with its confidence-aware dual-target supervision
and prediction-preserving gradient projection. Because Ours and Ours-MM+CA
are cumulative designs rather than one-factor variants, we complement
this comparison with dedicated ablations of gradient validation and
the three knowledge modalities. Figure $\ref{fig:exp-variant-comparison}$
further highlights that the most pronounced advantage of Ours lies
in worst-client accuracy rather than in attack-averaged mean performance.

\subsection{Mechanism Diagnostics}

Table $\ref{tab:exp-mechanism-diagnostics}$ summarizes final-round
diagnostics for Ours in the no-attack runs. Pseudo-label coverage
is the fraction of public examples whose weighted vote confidence
exceeds $\tau_{conf}=0.5$. Projection frequency is averaged over
honest clients and the two public updates in the final round.

The adaptive modality weights remain close to the configured prediction-dominant
prior, while retaining nonzero boundary and correlation contributions.
CIFAR-100 has substantially lower pseudo-label coverage than CIFAR-10,
which is expected in the more difficult 100-class label space. The
high CIFAR-100 prediction-projection frequency also shows that private
gradient validation is active rather than merely adding an unused
safeguard.

\section{Conclusions\label{sec:Conclusions}}

In this paper, we proposed a robust decentralized federated distillation
method for heterogeneous models under Byzantine attacks. The key motivation
is that clients with different model architectures cannot directly
exchange or compare model parameters, while Byzantine clients may
manipulate the predictions sent to different honest clients. Therefore,
instead of performing collaboration in the parameter space, the proposed
method exchanges predictions on shared unlabeled public data. Specifically,
each client first evaluates the received predictions in three modalities
of class prediction, boundary decision, and prediction correlation.
Unreliable clients are then filtered, and the retained clients are
weighted to construct three teachers in the three modalities for distillation.
Finally, a supervised gradient computed from private data is used
to validate the three distillation gradients. Conflicting prediction
and boundary gradients are removed, and conflicting relation gradients
are suppressed before the final model update.

Theoretically, we showed that our algorithm achieves a bounded Byzantine
influence on both distillation gradients of all modalities and final
client private gradients after cross-modality fusion, thereby ensuring
stable local optimization for honest clients under Byzantine distillation.
Empirically, experiments on CIFAR-10 and CIFAR-100 for clients with
heterogeneous model architectures, non-IID data, and multiple Byzantine
attacks demonstrated that the proposed method provides robust collaboration
among  clients and achieves improved local model prediction accuracy,
 showing that our method can reduce the impact of different malicious
predictions received by different clients. Across all evaluated attacks,
the method achieves the highest worst client-accuracy on both datasets,
demonstrating its application value for decentralized federated learning
in unreliable real-world scenarios where clients are exposed to receiver-specific
Byzantine messages of malicious predictions.

\section*{Acknowledgment}

This work is supported by the Science and Technology Development Fund
of Macao (FDCT) (Project $\#$0015/2023/RIA1), Queensland State Department
of Environment and Science under the Quantum Challenges 2032 Program
(Project \#Q2032001). The corresponding author is Hong Shen.

\bibliographystyle{IEEEtran}
\bibliography{reference}

\end{document}